\documentclass[preprint,12pt]{elsarticle}

\usepackage{amssymb}
\usepackage{amsmath}

\usepackage{color}
\usepackage{booktabs}
\usepackage{url} 
\usepackage{multirow} 
\usepackage{amsthm}

\usepackage{enumitem}
\usepackage{algorithm}
\usepackage{algpseudocode}
\usepackage{xcolor}
\usepackage{wrapfig}
\usepackage{makecell}

\newtheorem{theorem}{Theorem}
\newtheorem{definition}{Definition}
\newtheorem{lemma}{Lemma}

\journal{Elsevier}

\begin{document}

\begin{frontmatter}



\title{Unifying Invariant and Variant Features for Graph Out-of-Distribution via Probability of Necessity and Sufficiency}


\author[Address1]{Xuexin Chen}
\ead{im.chenxuexin@gmail.com}

\author[Address1,Address2]{Ruichu Cai\corref{cor1}}
\ead{cairuichu@gmail.com}
\cortext[cor1]{Corresponding author}

\author[Address1]{Kaitao Zheng}
\ead{1037452735@qq.com}

\author[Address1]{Zhifan Jiang}
\ead{468788700@qq.com}

\author[Address1]{Zhengting Huang}
\ead{zhengtinghuang68@gmail.com}

\author[Address3]{Zhifeng Hao}
\ead{haozhifeng@stu.edu.cn}

\author[Address4]{Zijian Li}
\ead{leizigin@gmail.com}

\address[Address1]{School of Computer Science, Guangdong University of Technology, Guangzhou 510006, China}
\address[Address2]{Peng Cheng Laboratory, Shenzhen 518066, China}
\address[Address3]{Mohamed bin Zayed University of Artificial Intelligence, Masdar City, Abu Dhabi}
\address[Address4]{College of Science, Shantou University, Shantou 515063, China}


\begin{abstract}
Graph Out-of-Distribution (OOD), requiring that models trained on biased data generalize to the unseen test data, has considerable real-world applications. 
One of the most mainstream methods is to extract the invariant subgraph by aligning the original and augmented data with the help of environment augmentation. 
However, these solutions might lead to the loss or redundancy of semantic subgraphs and result in suboptimal generalization. 
To address this challenge, we propose exploiting Probability of Necessity and Sufficiency (PNS) to extract sufficient and necessary invariant substructures. Beyond that,  we further leverage the domain variant subgraphs related to the labels to boost the generalization performance in an ensemble manner. 
Specifically, we first consider the data generation process for graph data. Under mild conditions, we show that the sufficient and necessary invariant subgraph can be extracted by minimizing an upper bound, built on the theoretical advance of the probability of necessity and sufficiency. 
To further bridge the theory and algorithm, we devise the model called  Sufficiency and Necessity Inspired Graph Learning (SNIGL), which ensembles an invariant subgraph classifier on top of latent sufficient and necessary invariant subgraphs,  and a domain variant subgraph classifier specific to the test domain for generalization enhancement. 
Experimental results demonstrate that our SNIGL model outperforms the state-of-the-art techniques on six public benchmarks, highlighting its effectiveness in real-world scenarios.
\end{abstract}




\end{frontmatter}



\section{Introduction}
\begin{figure*}[t!]
	\centering
	\includegraphics[width=\columnwidth]{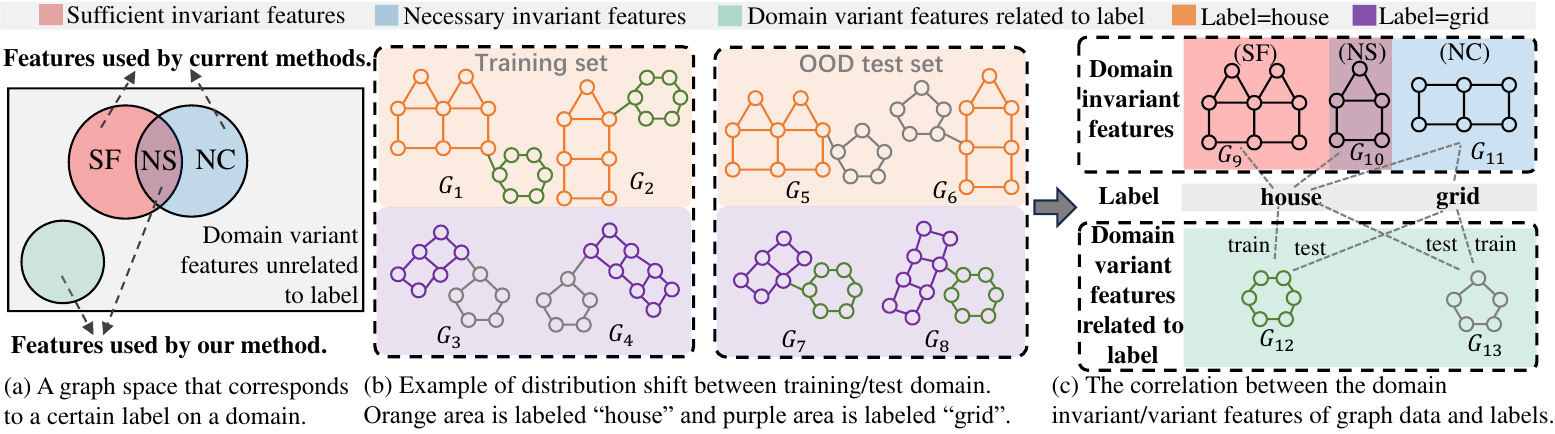} 
	\caption{Illustration of graph OOD methods with invariant subgraph learning, 
 (a) Existing methods exploit either sufficient (SF) or necessary (NC) invariant features due to the unclear trade-off between invariant graph feature space constraint and prediction loss. Our method exploits necessary and sufficient (NS) invariant features that achieve the optimal trade-off, and ensembles domain variant features to further improve generalization.
 (b) To illustrate the concepts of these features, we provide a toy example of a dataset with distribution shifts for graph classification.
 (c) The first line is all SF or NC invariant features for the ``house'' label, where NC features may lead to two labels and SF features may not be included in each ``house'' graph. In the third line, the correlation between variant features and labels will change with the domain where these features are located.}
 \label{fig:motivation}
\end{figure*}

Graph representation learning with \textbf{G}raph \textbf{N}eural \textbf{N}etworks (GNNs) has gained remarkable success in complicated problems such as intelligent transportation and the inverse design for \emph{polymers}~\cite{DBLP:journals/tits/RahmaniBBP23,CHEN2021100595}. 
Despite their considerable success, GNNs generally assume that the testing and training graph data are \textbf{i}ndependently sampled from the \textbf{i}dentical \textbf{d}istribution (IID). However, the validity of this assumption is often difficult to guarantee in real-world scenarios. 

To solve the Out Of Distribution (OOD) challenge of graph data, one of the most popular methods \cite{DBLP:journals/tkde/LiWZZ23,li2022graphde,zhao2020uncertainty,liu2023flood,sui2022causal} is to extract domain invariant  features of graph data for domain generalization (DG). Previously, Li et al. \cite{DBLP:journals/tkde/LiWZZ23} address
the OOD challenge by eliminating the statistical dependence between relevant and irrelevant graph representations. Since the spurious correlations lead to the poor generalization of GNNs, Fan et.al \cite{fan2023generalizing} leverage stable learning to extract the invariant components. Recently, several researchers have considered environment augmentation to extract invariant representations. Liu et.al \cite{liu2022graph} employ a  rationale-environment separation approach to address the graph-OOD challenge. Chen et.al \cite{chen2024does} further utilize environment augmentation to enhance the extraction of invariant features. Li \cite{DBLP:journals/corr/abs-2311-04837} employs data augmentation techniques to provide identification guarantees for the invariant latent variables.
In summary, these methods aim to achieve the invariant representation by 
balancing two objectives: 1) aligning the original and invariant feature spaces, and 2) minimizing the prediction error on the training data. 

Although existing methods with environmental augmentation have achieved outstanding performance in graph OOD, they can hardly extract optimal invariant subgraphs due to the difficulty of the trade-off between invariant alignment and prediction accuracy. 
To better understand this phenomenon, we provide a toy example of graph classification, where the ``house'' and ``grid'' labels are determined by the house-like and grid-like shapes respectively. 
Existing methods that balance the feature alignment restriction and the classification loss might result in two extreme cases. 
The first case is that GNNs put more weight on optimizing the classification loss. In this case, GNNs tend to extract latent sufficient invariant subgraphs. For example, $G_9$ and $G_{10}$ in Figure~\ref{fig:motivation}(c) are sufficient subgraphs of the ``house'' label, because these subgraphs with house-like shapes will only lead to the ``house'' label. However, $G_9$ is not an optimal invariant subgraph of ``house'', because it is not necessary, that is, other ``house'' samples may not contain $G_9$, such as $G_2$ in Figure~\ref{fig:motivation}(2), could result in classification errors.
The second case is that the alignment restriction is over-heavily strengthened. In this case, GNNs tend to extract latent necessary invariant subgraphs, that is, subgraph structures shared by most samples of the same class. For example, $G_{10}$ and $G_{11}$, as shown in Figure~\ref{fig:motivation}(c), are both necessary invariant subgraphs of the ``house'' label, because all graphs of this class contain $G_{10}$ and $G_{11}$. However, $G_{11}$ is not an optimal invariant subgraph because it is not sufficient, that is, $G_{11}$ may cause the model to be incorrectly classified into the ``grid'' label, because ``grid'' samples (such as $G_3$ and $G_8$) also contain subgraph $G_{11}$. 
Thus, it is vital to develop graph learning methods for achieving the optimal trade-off between prediction accuracy and invariant subspace constraints.


Based on the examples above, an intuitive solution to the graph OOD problem is to  extract the \emph{sufficient and necessary} invariant subgraphs for prediction. As shown in Figure~\ref{fig:motivation}(c), the subgraph $G_{10}$ is one of the optimal invariant subgraphs for the ``house'' label, because $G_{10}$ is shared by all the graphs of the ``house'' label (necessity), and $G_{10}$ can uniquely lead to the ``house'' label (sufficiency), which allows for the accurate prediction of ``house'' in any domain. However, in reality, not every label has necessary and sufficient invariant features, such as the ``grid'' label, since any of its necessary invariant subgraphs (e.g., $G_{11}$) may lead to the ``house'' label. 
To address this challenge, our key observation is that domain variant features are also helpful for predicting a certain domain if they are related to labels in that domain. Figure~\ref{fig:motivation} illustrates the ability of subgraph $G_{13}$ in predicting the ``house'' label in the test set. Under this intuition, we propose a Sufficiency and Necessity Inspired Graph Learning (SNIGL) method to exploit the domain variant subgraphs and sufficient and necessary invariant subgraphs for prediction. 
Specifically, to learn the necessary and sufficient invariant feature subspace, we resort to \emph{Probability of Necessity and Sufficiency} (PNS)~\cite{pearl2022probabilities} in causality, a notion for quantifying the probability that an event $A$ is a \emph{necessary and sufficient cause} of an event $B$. Based on PNS, we reduce this goal as a PNS optimization problem with respect to the invariant features. Since computing PNS is intractable since the counterfactual data is not available. We then propose a flexible PNS lower bound to solve this challenge under mild assumptions. In order to fill the gap of the lack of necessary and sufficient invariant features for some labels in the test domain, we propose a principled ensemble framework that combines invariant and variant features. Since the labels of the test domain are not available, we first train a biased domain variant feature classifier specific to the test domain through the pseudo-labeled data based on the training set and derive the calibration method under suitable assumptions. 
Finaly, our proposed SNIGL is validated on several mainstream simulated and real-world benchmarks for application evaluation. The impressive performance that outperforms state-of-the-art methods demonstrates the effectiveness of our method. 
\section{Preliminaries}
\subsection{Problem Setup}


\begin{wrapfigure}{l}{3.5cm} 
\includegraphics[width=0.2\textwidth]{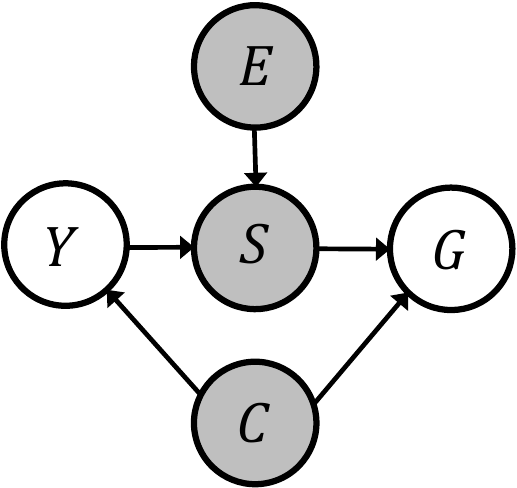}
\caption{The causal graph for domain generalization problem on Graph~\cite{DBLP:journals/corr/abs-1907-02893}. Grey and white nodes denote the latent and observed variables, respectively.}
\label{fig:causal_graph22}
\end{wrapfigure}
In this paper, we focus on domain generalization in graph classification where GNNs are trained on data from multiple training domains with the goal of performing well on data from unseen test domains. Formally, we consider datasets $\mathcal D = \{\mathcal{D}^e \}_{e \in \mathcal{E}^{\text{train}}}$ collected from $m$ different training domains or \emph{environments} $\mathcal{E}^{\text{train}} = \{1, ..., m\}$, with each dataset $\mathcal{D}^e=$ $\{(G^e_i, y^e_i)\}^{n_e}_{i=1}$ containing graph data pairs $(G^e_i, Y^e_i)$ sampled independently from an identical distribution $P(G, Y|E=e)$, where $G, E, Y$ denotes the variables of input graph, environment (i.e., domain indicator) and ground-truth label. $G_i = (\mathcal{V}_i, \mathcal{L}_i, X_i)$ denotes the $i$-th graph instance where  $\mathcal{V}_i$ is a set of nodes, $\mathcal{L}_i \subseteq \mathcal{V}_i \times \mathcal{V}_i$ is a set of edges and $X_i \in \mathbb R^{|\mathcal{V}_i| \times d^{\mathcal{V}_i}}$ is the node feature matrix. Each row of $X_i$ denote the $ d^{\mathcal{V}_i}$-dimensional feature vector for node $v \in \mathcal{V}_i$.
Let $\mathcal{G}$ and $\mathcal{Y}$ be the graph and label space. The goal of domain generalization on graphs is to learn an \emph{invariant GNN} $g_C \circ f_C$ that performs well on a larger set of possible domains $\mathcal{E}^{\text{all}} \supset \mathcal{E}^{\text{train}}$, where $f_C: \mathcal G \to \mathbb R^{d^c}$ is the encoder of invariant GNN that is used to extract domain invariant subgraph or representation $C = f_C(G)$ from each graph $G$ and $g_Y: \mathbb R^{d^c} \to \mathcal{Y}$ is the downstream classifier to predict the label $Y=g_C(C)$.

\textbf{Graph generation process.} 
Generating predictions that can generalize out of distribution requires understanding the actual mechanisms of the task of interest. Following previous works \cite{von2021self,chen2022invariance}, here we present a 
generation process of graph data behind the graph classification task, by inspecting the causalities among five variables: input graph $G$, ground-truth label $Y$, domain invariant subgraph $C$, domain variant subgraph $S$ and environment  $E$, where noises are omitted for simplicity.
Figure~\ref{fig:causal_graph22} illustrates the causal diagram, where each link denotes a causal relationship between two variables. $C \to G \leftarrow S$ indicates that the input graph $G$ consists of two disjoint components: the invariant subgraph $C$ and the variant or unstable subgraph $S$, according to whether they are affected by environment $E$, such as the orange and green components of $G_1$ in Figure X. 
Moreover, $C \to Y \to S$ indicates $C$ is \emph{partially informative} about $Y$, i.e., $(S, E) \not\perp Y \mid C$~\cite{chen2022invariance}. $C \to Y$ indicates the labeling process, which assigns labels $Y$ for the corresponding $G$ merely based on $C$. Taking the house classification example in Figure X again, $C$ of input graph $G_1$ can be the $G_9$, $G_{10}$ or $G_{11}$ (which one is better will be discussed in Section \ref{sec:pns}),  which perfectly explains why the graph is labeled as ``house''.

\subsection{Probability of Necessity and Sufficiency}\label{sec:pns}
Current OOD generalization methods on graph data mainly focus on learning only domain invariant features for label prediction. 
However, domain-invariant features can be divided into three categories, each of which has different effects on graph label prediction.

\textbf{1) Sufficient but unnecessary causes.} Knowing cause $A$ leads to effect $B$, but when observing effect $B$, it is hard to confirm $A$ is the actual cause. For example, the domain invariant feature $G_9$ can predict the label ``house'', but a graph with the label ``house'' label  might not contain this feature, such as $G_2$. 
\textbf{2) Necessary but insufficient causes.} Knowing effect $B$ we confirm the cause is $A$, but cause $A$ might not lead to effect $B$. For example, if the input graph does not contain the invariant feature $G_{11}$, then we can confirm that the label of this graph is not ``house''. However, graph $G_4$ with the ``grid'' label also has the same invariant feature $G_{11}$ as a ``house''. Thus invariant feature $G_{11}$ is not a stable feature to predict houses. 
\textbf{3)} \textbf{Necessary and sufficient causes.} Knowing effect $B$ we confirm the cause is $A$, and we also know that $A$ leads to $B$.  In the ``house'' and ``grid'' classification tasks, invariant feature  $G_{10}$ could be a necessary and sufficient cause. It is because $G_{10}$ allows humans to distinguish a ``house'' from a ``grid'', and when we know there is a ``house'', $G_{10}$ must exist. In conclusion, the accuracy of predictions based on domain-invariant features with sufficient necessary information will be higher than those based on other types of invariant features.

In order to learn sufficient and necessary domain invariant features of the input graph, we resort to the concept of \emph{Probability of Necessity and Sufficiency} (PNS) \cite{pearl2022probabilities}, which is defined as follows.
\begin{definition}\label{def:pns} (Probability of necessity and sufficiency (PNS) \cite{pearl2022probabilities}) Let the specific values of domain invariant variable $C$ and label $Y$ be $c$ and $y$. The probability that $C = c$ is the necessary and sufficiency cause of $Y = y$ is
\begin{equation}\label{equ:pns_cy}
\small
\begin{aligned}
\operatorname{PNS}(C\!=\!c, Y\!=\!y) &:=\underbrace{P(Y_{do(C=c)}\!=\!y \mid C \!\ne\! c, Y \!\ne\! y)}_{\text {sufficiency }} P(C \ne c, Y \neq y) \\
& +\underbrace{P(Y_{do(C \ne c)} \!\ne\! y \mid C\!=\!c, Y\!=\!y)}_{\text {necessity }} P(C\!=\!c, Y\!=\!y).
\end{aligned}
\end{equation}
\end{definition}
In the above definition, the notion $P(Y_{do(C = c)} = y|C\ne c, Y \ne y)$ means that we study the probability of $Y \ne y$ when we force the variable $C$ to be a value $do(C = c)$ (i.e., perform the do-calculus~\cite{pearl2000models}) given a certain factual observation $Y \ne y$ and $C \ne c$. The first and second terms in PNS correspond to the probabilities of sufficiency and necessity, respectively. Event $C$$=$$c$ has a high probability of being the sufficient and necessary cause of event $Y$$=$$y$ when the PNS value is large. 

\section{Theory: Unifying Invariant and Variant Features for Graph OOD via PNS}
In this section, motivated by the fact that not necessary or not sufficient invariant features may be harmful to domain generalization on the graph, we present our main theoretical result which shows how to unify invariant and variant subgraphs for graph OOD via PNS. 
We begin by describing how to extract necessary and sufficient invariant information by identifying the subspace of the variable $C$. 
Since not every graph contains this invariant subgraph, 
we describe how to alleviate this problem by reconstructing $P(Y|C,S)$ from $P(Y|C)$ and $P(C, S)$ which involves exploiting domain variant features to enhance domain generalization.
\subsection{Necessary and Sufficient Invariant Subspace Learning}
In this section, supposing that we have already identified the domain invariant subgraph $C$, we analyze the problem of extracting the necessary and sufficient invariant features about $C$ from $G$. 
We first reduce it to an optimization problem for PNS, i.e., identify the subspace of $C$ with the largest PNS with respect to $C$ and $Y$ given a  graph $G$. However, PNS is usually intractable because counterfactual data are not available. We show this issue can be solved exactly by deriving the lower bound of PNS for optimization.

We now formalize the two key assumptions underlying our approach. These assumptions below  will help us  derive the lower bound of PNS based on conditional probability.
\begin{definition}\label{def:exogeneity} (Exogeneity~\cite{pearl2000models})
Variable $C$ is exogenous relative to variable $Y$ if $C$ and $Y$ have no common ancestor in the graph generation process. 
\end{definition} 
\begin{definition}\label{def:consistency} (Consistency~\cite{pearl2000models,hernan2010causal})
    If variable $C$ is assigned the value $c$, then the observed outcome $Y$ is equivalent to its outcome $Y_{do(C = c)}$ of intervention; i.e., if $C = c$, then $Y = Y_{\text{do}(C=c)}$.
\end{definition}
We will discuss the roles of these assumptions after stating our main result. 

\textbf{Reduction to the optimization problem for PNS.} 
Suppose we have used the training data to learn the invariant GNN $g_C \circ f_C$ and thus know $P(C|G) \approx f_C(G)$ and $P(Y|C) \approx g_C(C)$, and recall that our goal is to predict $Y$ using the necessary and sufficient invariant features about $C$ from $G$. Thus, our task becomes to reconstruct $P(C|G)$, that is, to find a subspace in $C$ that contains all necessary and sufficient causes for the label $Y$ of $G$. 
By Definition~\ref{def:pns}, a trivial solution is to find the subspace of $C$ for a given $G$ that maximizes PNS. However, computing the intervention probability is a challenging problem since collecting the counterfactual data is difficult, or even impossible in real-world scenarios, it is not feasible to optimize PNS directly. Fortunately, motivated by probabilities of causation theory~\cite{pearl2022probabilities}, we show that the lower bound of PNS can be theoretically identified by the observation data under proper conditions.
\begin{theorem}
\label{thm:bound} (Lower bound of PNS).
Consider two random variables $C$ and $Y$. If exogeneity and consistency assumptions hold,  then the lower bound of $\text{PNS}(C=c,Y=y)$ is as follows:
\begin{equation}\label{equ:lower}
\small
    \max \Big(\frac{ P(Y = y|C =c) - P(Y=y)}{1 - \mathbb E_G [P(C=c|G)]}, 0\Big) \leq \text{PNS}(c,y)
\end{equation}
\end{theorem}
\textbf{Proof Sketch of Theorem \ref{thm:bound}.} We begin by deriving the lower bound of the probability $P(A, B)$ for arbitrary events $A$ and $B$ based on  \emph{Bonferroni's inequality}. Then, by consistency assumption, we express $P(A, B)$ in the form of PNS (Definition \ref{def:pns}). Finally, by exogeneity assumption, we use conditional probability to identify intervention probability $P(Y_{do(C = c)} = y)$$=$$P(Y=y|C=c)$. This yields the lower bound of PNS as shown in Eq.~\ref{equ:lower} (full proof is provided in \ref{sec:proof_bound}). 

Theorem \ref{thm:bound} inspires us that when counterfactual data is unavailable, we can reconstruct $P(C|G)$ by maximizing the lower bound of PNS to extract the necessary and sufficient invariant feature subspace in $C$. The following section will present a reconstruction objective for $P(C|G)$ based on the lower bound of PNS.

\subsubsection{PNS Risk}
Based on the PNS lower bound (Eq.~\ref{equ:lower}), this section presents the PNS risk which shows how to estimate the subspace of variable $C$. 
The objective decreases when the subspace of variable $C$ contains less necessary and sufficient information. 
The objective is based on the lower bound of PNS (Eq.~\ref{equ:lower}). 
Formally, given the distribution of invariant features $P(C|G)$, we use the notation $P_{\theta^c}(C|G)$ to present the estimated reconstructed distributions which are parameterized by $\theta^c$, and the support of $P_{\theta^c}(C|G)$ denote the subspace of $C$. We can adapt $P_{\theta^c}(C|G)$ to minimize the following PNS risk with respect to $\theta_c$
on environment $e$:
\begin{equation}\label{equ:pns_objv1}
\small
\begin{aligned}
     &R^e_{\text{NS}}(P_{\theta^c}(C|G), P(Y|C))=\\
     &\mathbb E_{G', y \sim P(G, Y|E=e)} \mathbb E_{c\sim P_{\theta^c}(C|G')} \Big[\frac{ P(Y\!=\!y|E\!=\!e) - P(Y \!=\! y|C \!=\!c)}{1 - \mathbb E_G [P_{\theta^c}(C\!=\!c|G)|E\!=\!e]}\Big],
\end{aligned}
\end{equation}
where $\mathbb E_{c \sim P_{\theta^c}(C|G)} [ \frac{P(Y = y|C =c) - P(Y=y|E=e)}{1 - \mathbb E_G [P_{\theta^c}(C=c|G)|E=e]})]$ represents the expectation of the lower bound of PNS (Eq.~\ref{equ:lower}) over $P_{\theta^c}(C|G)$ given a graph $G$. 
Note that the identification of the subspace of $C$ of $G$ containing only necessary and sufficient features can  be confirmed when above expectation of the lower bound equal to $1$. Otherwise, the learned subspace may contain not sufficient or not necessary information. 

In general, although Eq.~\ref{equ:pns_objv1} is capable of learning subspace for each input graph $G$ with values of PNS as large as possible, for some graphs $G$, their PNS may be less than 1. One main reason is that some graphs may not have sufficient and necessary invariant features, such as $G_3$ and $G_4$ in Figure~\ref{fig:motivation}. 
This usually leads to significant performance drops or even complete failure. 
To mitigate the negative impact of insufficient or unnecessary invariant features, 
the following section demonstrates that, we can also use domain variant subgraphs $S$ that are related to the label $Y$.

\subsection{Ensemble Learning with Variant Features}
In this section, to mitigate the negative impact of not sufficient or not necessary invariant subgraphs on prediction, 
our key observation is that if the domain variant subgraph $S$ is related to the label $Y$, then we can improve prediction accuracy by ensembling predictions based on unstable or domain variant subgraphs specific to the test domain $e$. Based on this observation, 
we describe a boosted joint GNN in the test domain $e$ as a combination of an invariant GNN $g_C \circ f_C$ 
and an \emph{unstable GNN} $g_{S,e} \circ f_S$ training on the domain variant subgraphs on test domain $e$. 
The unstable GNN $g_{S,e} \circ f_S$ is composed of a domain variant subgraph extractor $f_S: \mathcal{G} \to \mathbb (0, 1)^{n \times n}$ across domains, and a domain-specific classifier $g_{S,e}$ on test domain $e$, where $n$ is the number of nodes of $G$. Then, the boosted joint GNN is denoted by
\begin{equation}\label{equ:boost}
\small
     Y = \text{COMBINE}((g_C (f_C(G)), g_{S,e}(f_S(G))).
\end{equation}
The components $g_C \circ f_C$ and $f_S(G)$ of Eq.~\ref{equ:boost} can be obtained through  our PNS risk (Eq.~\ref{equ:pns_objv1}) and conditional independence test penalty, which will be introduced in Section \ref{sec:train_domain}. However, it is non-trivial to derive a principled choice for the $\text{COMBINE}$ function and learn the classifier $g_{S,e}$ specific to the test domain $e$ because we don’t have labels in the test domain $e$. 
Building upon the work of Cian Eastwood et al.~\cite{DBLP:conf/nips/Eastwood0NPKS23}, we  adopt 
\emph{pseudo-labels} $\hat Y_i = \arg\max_{j} g_C(f_C(G_i))$
denoting the $j$-th entry of the prediction for $G_i$. 
In other words, we need to analyze the problem of using these pseudo-labels to learn $g_{S,e}$ in the test domain $e$ without true labels. 
Cian Eastwood et al. reduce this to a ``marginal problem'' in probability theory, i.e., the problem of identifying a joint distribution based on information about its marginals. 

\textbf{Reduction to the marginal problem.} Suppose we have enough unlabeled data from test domain $e$ to learn $P(C, S|E=e)$, our goal is to predict $Y$ from $(C, S)$ in test domain $e$. Our key observation is that if we can decompose $P(Y|C,  S, E)$
into two terms,  $P(Y| C)$ and $P(Y|S, E)$, then we can utilize  $P(Y|C, S)$ to achieve optimal prediction of $Y$ from $(C, S, E)$. Thus, our task is  broken down into the reconstruction of $P(Y| C,  S)$ from $P(Y| C)$ and $P(Y|S, E)$. With the help of a marginal problem theory~\cite{DBLP:conf/nips/Eastwood0NPKS23}, the theorem below demonstrates that, under our assumption of causal relationships between variables in Figure~\ref{fig:causal_graph22}, we can exactly recover these terms. To simplify notation, we assume the label $Y$ is binary and leave the multi-class extension for  \ref{sec:multi_class}.
\begin{theorem}\label{lma:boost}
Consider variables $C$, $S$, $Y$, and $E$, where $Y$ is binary ($Y\!\in\!\{0, 1\}$). If $S \perp C|Y$ and $C \not\perp Y$, then the distribution $P(Y\!=\!1|C,S,E)$ can be decomposed into three components: $P(Y| C)$, $P(Y| S, E)$, and $P(Y\!=\!1)$. Specifically, if  $\hat Y \!\sim\! \text{Bernoulli}(P(Y\!=\!1|C))$ is a pseudo-label, then we have 
\begin{equation}\label{equ:calibrate}
\small
\begin{aligned}
    &P(Y\!=\!1| C,  S, E) \!=\! \sigma(\text{logit}(P(Y\!=\!1|  C)) \!+\! \text{logit}(P(Y\!=\!1|  S,E)) \!+\! \text{logit}(P(Y\!=\!1)))\\
    &P(Y=1| S, E) = \frac{P(\hat Y =1| S, E) + P(\hat Y = 0|Y=0) - 1}{P(\hat Y = 0|Y=0) + P(\hat Y = 1|Y=1) - 1},\\
    &P(\hat{Y}\!=\!1|Y\!=\!1) \!=\! \frac{\mathbb E_{C} [ P(Y\!=\!1| C)^2]}{\mathbb E_{C} [P(Y\!=\!1 | C)]}, \quad    P(\hat{Y}\!=\!0|Y\!=\!0) \!=\! \frac{\mathbb E_{C} [ P(Y\!=\!0| C)^2]}{\mathbb E_{C} [P(Y\!=\!0 | C)]}
\end{aligned}
\end{equation}
\end{theorem}
Intuitively, the second row of Eq.~\ref{equ:calibrate} shows how to calibrate the distribution of pseudo labels $P(\hat{Y}|S,E)$ to the distribution of true labels $P(Y|S,E)$, in other words, it provides a way for us to obtain a domain-variant subgraph classifier $g_{S,e}$ specific to the test domain $e$ even without true labels. The first row of Eq.~\ref{equ:calibrate} shows how the predictions of invariant subgraphs $C$ and variant subgraphs $S$ are ensembled, justifying the combination function
\begin{equation}\label{equ:combine_prob}
\small
\begin{aligned}
    &\text{COMBINE}(P(Y=1|C), P(Y=1|S, E)) =\\
    &\sigma(\text{logit}(P(Y=1|  C)) + \text{logit}(P(Y=1|  S,E)) + \text{logit}(P(Y=1)))
\end{aligned}
\end{equation}
Note that according to Theorem~\ref{lma:boost}, if $S$ and $Y$ are independent, then integrating domain variant features into the model will not actually help improve predictions.

\textbf{Proof Sketch of Theorem \ref{lma:boost}.} 
Our proof consists of two parts. The first part is to calibrate $P(\hat Y=1|S)$. We first use the law of total probability to insert the variable $Y$ to $P(\hat Y=1|S)$, i.e., 
\begin{equation*}
\small
    P(\hat Y\!=\!1|S) \!=\! P(Y\!=\!1|S) \cdot P(\hat Y \!=\!1| S,Y\!=\!1) + P(Y\!=\!0|S) \cdot P(\hat Y \!=\!1|S,Y\!=\!0)
\end{equation*}
and according to our assumption of causal relationships between variables in Figure~\ref{fig:causal_graph22}, we have $S \perp C | Y$, and since $\hat Y$ is determined by $C$, we can eliminate $S$ from $P(\hat Y|S,Y)$, as follows.
\begin{equation*}
\small
    P(\hat Y=1|S) = P(Y=1|S) \cdot P(\hat Y =1|Y=1) + P(Y=0|S) \cdot P(\hat Y =1|Y=0)
\end{equation*}
Hence, re-arranging the above equality will give us the conditional distribution $P(Y=1|S)$ (Eq.~\ref{equ:calibrate}). 
The second part of the proof is to decompose $P(Y|C,S,E)$ into two terms, $P(Y=1|S,E)$ and $P(Y=1|C)$. We first calculate the \emph{Odds} of probability $P(Y|C,S,E)$, that is,
\begin{equation}\label{equ:sdfafg}
\small
    \frac{P(Y=1|C, S, E)}{P(Y=0|C, S, E)}  = \frac{P(Y=1)P(C, S, E|Y=1)}{P(Y=0)P(C, S, E|Y=0)}.
\end{equation}
Since $C \perp  \{S, E\}| Y$, we have $P(C, S,E|Y) = P(C|Y)P(S,E|Y)$. Substituting this equality into Eq.~\ref{equ:sdfafg}, we get the logit of $P(Y=1|C,S,E)$ as follows.
\begin{equation}
\small
\begin{aligned}
    &\text{logit}(P(Y=1|C, S, E))  =\\
    &\text{logit}(P(Y=1)) + \text{logit}(P(Y=1|S,E)) + \text{logit}(P(Y=1|C))
\end{aligned}
\end{equation}
Since the sigmoid $\sigma$ is the inverse of logit, we can write the distribution $P(Y=1|C, S,E)$  in terms of the conditional distributions $P(Y=1|S,E)$ and $P(Y=1|C)$ and the marginal $P(Y=1)$.
Full proof is provided in \ref{sec:proof_boost}.

\textbf{Remark 1.} Eq.~\ref{equ:boost} utilizes almost all the information in the input graph $G$ through separate modeling on the training domain and the test domain $e$, because we need to extract sufficient and necessary domain invariant features from the training domain 
and domain varying features specific to the test domain $e$ for prediction. If we do not model separately, i.e., only train a single model, it is difficult for this model to extract domain-specific features from the training domain that are specific to the test domain $e$.

\section{Algorithm: Sufficiency and Necessity Inspired Graph Learning}\label{sec:alg}
In this section, we will leverage the above theoretical results and propose a data-driven method called Sufficiency and Necessity Inspired Graph Learning (SNIGL) to employ necessary and sufficient invariant subgraphs and domain variant subgraphs specific to the test domain $e$ for domain generalization. 
On the training domains, we describe learning an invariant GNN that extracts necessary and sufficient invariant features and an unstable GNN that extracts domain variant features.
On the test domain, we then describe how to combine these features to enhance the performance of domain generalization. 
\subsection{Training domains: Learning necessary and sufficient invariant and variant subgraphs.}\label{sec:train_domain}
Our goal in the training domains is to learn three modules $f_{\theta^c}$, $g_{\phi^c}$ and $f_{\theta^s}$ parameterized by $\theta^c$, $\phi^c$ and  $\theta^s$, respectively. 
The first two modules are the estimated necessary and sufficient invariant subgraphs extractor $f_{\theta^c}$ and its downstream classifier $g_{\phi^c}$, which form the invariant GNN $g_{\phi^c} \circ f_{\theta^c}$. The third module is the variant subgraphs extractor $f_{\theta^s}$ across domains that will be employed to adapt a classifier $g_{\phi^{s,e}}$ specific to the test domain $e$, which form the unstable GNN $g_{\phi^{s,e}} \circ f_{\theta^s}$ in the test domain $e$.

To achieve these learning goals, 
let $R_{\text{INV}}$ denote the risk of learning invariant GNNs proposed by existing methods, $R_{\text{joint}}$ denotes the risk (e.g., cross entropy) of the joint predictions $\text{COMBINE}(g_{\phi^c} \circ f_{\theta^c}, g_{\phi^{s,e'}} \circ f_{\theta^s})$ (Eq.~\ref{equ:combine_prob}) of the invariant GNN $g_{\phi^c} \circ f_{\theta^c}$ and unstable GNN $g_{\phi^{s,e'}} \circ f_{\theta^s}$ specific to the training domain $e'$, and $R_{CI}$ denote the penalty encouraging conditional independence $C \perp S | Y$. 
Technologically, we estimate modules $f_{\theta^c}$ and $g_{\phi^c}$ in two steps. First, we encourage the estimated invariant GNN $g_{\phi^c} \circ f_{\theta^c}$ to learn an invariant feature space through the invariant risk. 
Second, use our proposed PNS risk $R_{\text{NS}}$ to further encourage $g_{\phi^c} \circ f_{\theta^c}$ to find necessary and sufficient invariant subgraph space. 
In the meanwhile, we estimate the variant feature extractor $f_{\theta^s}$ in three steps. First, we introduce its downstream classifiers $g_{\phi^{s,e}}, ~e \in \mathcal{E}^{\text{train}}$ on each training environment to form the estimated unstable GNNs $g_{\phi^{s,e}} \circ f_{\theta^s}$ on training domains. 
Second, we use the COMBINATION function (Eq.~\ref{equ:combine_prob}) to jointly train the above invariant GNN and unstable GNNs. 
Third, given that the condition of the COMBINATION function is that $C \perp S|Y$, we add a penalty term $R_{CI}$ to the model. 
In summary, the aforementioned three steps can be formalized as the  following objective.
\begin{equation}\label{equ:pns_objv2}
\small
\begin{aligned}
     \min_{\theta^c, \theta^s, \phi^c, \Phi^s} \sum_{e \in \mathcal{E}_{\text{train}}}  &R^e_{\text{NS}}(f_{\theta^c}, g_{\phi^c}) +  R_{\text{INV}}^e(f_{\theta^c}, g_{\phi^c})\\
     &+  R_{\text{Joint}}^e(\text{COMBINE}(g_{\phi^c} \circ f_{\theta^c}, g_{\phi^{s,e}} \circ f_{\theta^s})) + \lambda \cdot R^e_{\text{CI}}(f_{\theta^c}, f_{\theta^s}),
\end{aligned}
\end{equation}
where $\Phi^s = \{\phi^{s,e}\}_{e \in \mathcal{E}_{\text{train}}}$ is the set of parameters of classifiers $g_{\phi^{s,e}}$ specific to different training environments and $\lambda \in [0, \infty)$ is its regularization hyperparameters. In practice, we use HSCIC~\cite{park2020measure} as $R_{CI}$, and the objective of CIGA as $R_{\text{INV}}$. Furthermore, we found that in practice it is unnecessary to set another hyperparameter to control the relative weights of the PNS risk $R_{\text{NS}}$, the invariant risk $R_{\text{INV}}$, and the joint risk $R_{\text{Joint}}$. Sections \ref{sec:pcg}-\ref{sec:pccg} show the implementation details of the components of Eq.~\ref{equ:pns_objv2}.

\subsubsection{Implementation of Necessary and Sufficient Invariant
Subgraphs Extractor $f_{\theta^c}$ and Domain Variant
Subgraphs Extractor $f_{\theta^s}$}\label{sec:pcg}
We employ the following implementation of $f_{\theta^c}$ and $f_{\theta^s}$ to generate an invariant subgraph $c$ and variant subgraph $s$, which can be formalized as follows. 
We first assume that variables $C$ and $S$ follow multivariate Bernoulli distributions with the parameters $\hat B^C$$\in$$(0, 1)^{n \times n}$ and $\hat B^S$$\in$$(0, 1)^{n \times n}$, i.e., $C \sim \text{Bernoulli}(\hat B^C)$ and $S \sim \text{Bernoulli}(\hat B^S)$, where $n$ is the number of nodes in graph $G$. Technologically, we estimate the parameters $\hat B^C$ and $\hat B^S$ of the  Bernoulli distribution in three steps. First, we use two layer graph neural network (GNN) to generate the node embeddings $Z^c$ and $Z^s$. Second, we calculate the parameter matrices $\hat B^C$ and $\hat B^S$, which denote the probability of the existence of each edge of $C$ and $S$. Third, we sample $C$ and $S$ from the estimated distributions. In summary, the aforementioned three steps can be formalized as follows:
\begin{equation}\label{equ:pcg}
\small
\begin{aligned}
    &Z^c \!=\! f_{\theta^c}(G) \!:=\! \text{GNN}(G; \theta^c)\!\in\! \mathbb R^{n \times d^c}, \hat{B}_C \!=\! \sigma(Z^c  Z^{cT}), C \!\sim\! \text{Bernoulli}(\hat{B}^C),\\
    &Z^s \!=\! f_{\theta^s}(G) \!:=\! \text{GNN}(G; \theta^s)\!\in\! \mathbb R^{n \times d^s}, \hat{B}^S \!=\! \sigma(Z^s  Z^{sT}), S \sim \text{Bernoulli}(\hat{B}^S),
\end{aligned}
\end{equation}
where $\text{GNN}(\cdot; \theta)$ denotes a particular graph neural network (e.g., GCN~\cite{DBLP:conf/iclr/KipfW17}) parameterized by $\theta$; $\sigma$ is the Sigmoid function. Here we employ Gumbel-Softmax~\cite{DBLP:conf/iclr/JangGP17} to sample $C$ and $S$.

\subsubsection{Implementation of Invariant Classifier $g_{\phi^c}$ and Variant Classifier $g_{\phi^{s,e}}$}\label{sec:pyc}
To estimate the predicted class probabilities, we use the $\text{READOUT}$ (e.g., mean)~\cite{DBLP:conf/iclr/XuHLJ19} function aggregates node embeddings to obtain the entire graph’s representation, and employ different three layer Multilayer Perceptron (MLP) with $\text{SOFTMAX}$ function activation to estimate the probabilities for each different environment, as follows:
\begin{equation}\label{equ:prob}
\small
\begin{aligned}
    &P(Y=y|C=c) = g_{\phi^c}(c)_y := \text{SOFTMAX}(\text{MLP}(\text{READOUT}(Z_C)))_y,\\
    &P(Y=y|S=s, E=e) = g_{\phi^{s,e}}(s)_y := \text{SOFTMAX}(\text{MLP}_e(\text{READOUT}(Z_S)))_y,
\end{aligned}
\end{equation}
where $g(\cdot)_y$ denotes the $y$-th entry of the prediction and $\text{MLP}_e$ represents an MLP specific to environment $e$.

\subsubsection{Implementation of $P_{\theta^c}(C=c|G)$ and $P(Y=y|E=e)$}\label{sec:pccg}
To estimate the component $P_{\theta^c}(C=c|G)$ of PNS risk $R^e_{\text{NS}}$ (Eq.~\ref{equ:pns_objv1}), our main idea is to draw samples from $P_{\theta^c}(C|G)$ and estimate the probability $P_{\theta^c}(C=c|G)$ by counting how many samples are isomorphic to $c$. Since graph isomorphism is an NP-hard problem, we simplify this problem by calculating the similarity between their graph representations. 
First, based on our implementation of $P_{\theta^c}(C|G)$ (Section~\ref{sec:pcg}), we draw $k$ subgraphs $c_1, ..., c_k$ from $P(C|G)$. 
Second, referring to Eq.~\ref{equ:pcg}-\ref{equ:prob}, we feed these subgraphs $c_1, ..., c_k$, and $c$ into the $\text{GNN}(\cdot;\theta^c)$ and use the $\text{READOUT}$ function to obtain their graph representations.
Third, compute the similarity (we adopt the inner product) between $c$ and $c_1,...,c_k$ in terms of their representation and average these similarities. The three steps can be summarized as follows:
\begin{equation}
\small
\begin{aligned}
    &P_{\theta^c}(C=c|G) \approx \frac{1}{k} \sum^k_{i=1} \sigma(\tilde Z^{(c_i)T} \tilde Z^{(c)}), \quad c_1, ..., c_k \overset{\text{i.i.d.}}{\sim} P_{\theta^c}(C|G),\\
    &\tilde Z^{(c_k)} \!=\! \text{READOUT}(\text{GNN}(G^{c_k}; \theta^c)), 
    \tilde Z^{(c)} \!=\! \text{READOUT}(\text{GNN}(G^{c}; \theta^c)), 
\end{aligned}
\end{equation}
where $G^{c_k}$ denotes the graph data associated with subgraph $c_k$, $\tilde{Z}^{(c_k)} \!\in\! \mathbb R^{d^c}$ and $\tilde{Z}^{(c)} \!\in\! \mathbb R^{d^c}$ denote the graph representation of $c_k$ and $c$, respectively. Finally, We estimate the component $P(Y=y|E=e)$ of PNS risk $R^e_{\text{NS}}$ using the following empirical distribution.
\begin{equation}
\small
    P(Y=y|E=e) \approx \frac{1}{|\mathcal{D}^e|} \sum_{ y^e_i \in \mathcal{D}^e} \mathbb I(y^e_i=y),
\end{equation}
where $\mathbb I(\cdot)$ is an indicator function, i.e., $\mathbb I(x) = 1$ if the statement $x$ is true and 0 otherwise.
\subsection{Test-domain Adaptation Without Labels.}\label{sec:adaptation}

Given the trained invariant GNN $g_{\phi^c} \circ f_{\theta^c}$ and the domain varying feature extractor $f_{\theta^s}$, our goal in the test domain is to adapt a classifier $g_{\phi^{s,e}}$ specific to test domain $e$ learned on top of trained $f_S$, so that we can make optimal use of the domain variant features extracted from $f_{\theta^s}$. By Theorem~\ref{lma:boost}, this goal can be achieved through the following three steps. 
First, given the unlabelled test domain data $\{G_i\}^{n}_{i=1}$, compute soft pseudo-labels $\{\hat Y\}^n_{i=1}$ with
\begin{equation}
\small
    \hat{y}_i = \arg\max_y g_{\phi^c}(f_{\theta^c}(G_i)).
\end{equation}
Second, letting $\ell\!:\! \mathcal{Y} \!\times\! \mathcal{Y} \!\to\! \mathbb R$ be a loss function (e.g., cross entropy), fit the biased classifier $\hat g_{\phi^{s,e}}(S)$ specific to test domain $e$ on pseudo-labelled data $\{(s_i = f_{\theta^{s,e}}(G_i), \hat{y}_i)\}$ with
\begin{equation}\label{equ:bggggsseee}
\small
    \min_{\phi^{s,e}} \frac{1}{n}\sum_{s_i, \hat{y}_i} \ell(\hat g_{\phi^{s,e}}(s_i), \hat{y}_i).
\end{equation}
Third, through optimizing Eq.~\ref{equ:bggggsseee}, we are given the trained biased classifier $\hat g_{\phi^{s,e }}(S)$ specific to the test domain $e$. By Theorem~\ref{lma:boost}, we calibrate $\hat g_{\phi^{s,e }}(S)$ as follows:
\begin{equation}
\small
     g_{\phi^{s,e}}(s_i) = \frac{\hat g_{\phi^{s,e}}(s_i) + \hat{\epsilon}_0- 1}{\hat{\epsilon}_0 + \hat{\epsilon}_1 -1}, 
\end{equation}
\begin{equation}
\small
     \hat{\epsilon}_1 = \frac{\sum_{G_i} g_{\phi^c}(f_{\theta^c}(G_i))^2}{\sum_{G_i} g_{\phi^c}(f_{\theta^c}(G_i))},\hat{\epsilon}_0 = \frac{\sum_{G_i} 1 - g_{\phi^c}(f_{\theta^c}(G_i))^2}{\sum_{G_i} 1 - g_{\phi^c}(f_{\theta^c}(G_i))}
\end{equation}
Finally, by the first line of Eq.~\ref{equ:calibrate}, we combine the prediction between the trained invariant GNN $g_{\phi^c} \circ f_{\theta^c}$ and the trained calibrated unstable GNN $g_{\phi^{s,e}} \circ f_{\theta^s}$ specific to the test domain $e$.

\section{Experiments}\label{sec:exp}
In this section, we evaluate the effectiveness of our proposed SNIGL model on both synthetic and real-world datasets by answering the following questions. 
\begin{itemize}[leftmargin=*]
    \item \textbf{Q1:} Whether the proposed SNIGL can outperform existing state-of-the-art methods in terms of model generalization.
    \item \textbf{Q2:} Can the proposed PNS risk learning necessary and sufficient invariant latent subgraphs well?
    \item \textbf{Q3:} Do ensemble strategies that exploit domain-varying subgraphs benefit model performance?
    \item \textbf{Q4:} What are the learning patterns and insights from SNIGL training? In particular, how do invariant or variant subgraphs help improve generalization?
\end{itemize}
\subsection{Experimental Setup}
\begin{table}[htbp]
  \centering
  \caption{Statistics of the datasets. 
  }
  \scalebox{0.7}{
    \begin{tabular}{llccccc}
    \toprule
    \multicolumn{1}{l}{Category} & Dataset & Shift source & \# Graphs & \# Nodes (Avg.) & \# Edges (Avg.)  & Metric \\
\hline
\hline
    \multirow{3}[4]{*}{\makecell[l]{Synthetic\\datasets}} & SP-Motif-Mixed & motif\&feature & 30,000 & 13.8  & 39.9       & ACC \\
\cmidrule{2-7}          & \multicolumn{1}{l}{\multirow{2}[2]{*}{GOOD-Motif}} &motif & 24,000 & 20.9  & 56.9       & ROC-AUC \\
          &       & size  & 24,000 & 32.7  & 86.8       & ROC-AUC \\
\hline
\hline
    \multirow{5}[4]{*}{\makecell[l]{Real-world\\datasets}} & \multicolumn{1}{l}{\multirow{2}[2]{*}{GOOD-HIV}} & scaffold & 32,903 & 25.3  & 54.4       & ROC-AUC \\
          &       & size  & 32,903 & 24.9  & 53.6       & ROC-AUC \\
\cmidrule{2-7}          & OGBG-Molsider & scaffold & 1,427  & 33.6  & 70.7      & ROC-AUC \\
          & OGBG-Molclintox & scaffold & 1,477  & 26.2  & 55.8       & ROC-AUC \\
          & OGBG-Molbace & scaffold & 1,513  & 34.1  & 73.7       & ROC-AUC \\
    \bottomrule
    \end{tabular}%
    }
  \label{tab:Statistics}%
\end{table}%
\subsubsection{Dataset}
To evaluate the effectiveness of our proposed SNIGL, we utilize six public benchmarks under different distribution shifts for graph classification tasks, including two synthetic datasets Spurious-Motif-Mixed~\cite{DBLP:conf/iclr/WuWZ0C22} and GOOD-Motif~\cite{gui2022good}, as well as four real-world datasets GOOD-HIV~\cite{gui2022good}, OGBG-Molsider, OGBG-Molclintox and OGBG-Molbace~\cite{hu2020open}. Table~\ref{tab:Statistics} summarizes the statistics of seven datasets.
\begin{itemize}[leftmargin=*]
    \item \textbf{Spurious-Motif-Mixed (SP-Motif-Mixed)}~\cite{chen2022learning} and \textbf{GOOD-Motif}~\cite{gui2022good} are synthetic datasets constructed based on BAMotif~\cite{ying2019gnnexplainer}. On one hand, the SP-Motif-Mixed dataset is created based on three motifs (House, Cycle, Crane) and three base graphs (Tree, Ladder, Wheel). Its distribution shifts are a mixture of the following two sources. The first shift is injected by the structure. Specifically, for a given bias $b$, a particular motif (e.g., house) is attached to a particular base graph (e.g., tree) with probability $b$, while for the other subgraphs, the motif is attached with probability $(1 - b)/2$ (e.g., house-ladder, house-wheel). 
    The second shift comes from the node attributes. Specifically, for a given bias $b$, the node features of the graph with label $y$ are also assigned $y$ with probability $b$, and the probability of being assigned to other labels is $(1-b)/ 2$.  In our experiment, we closely followed the literature of DIR~\cite{DBLP:conf/iclr/WuWZ0C22} and CIGA~\cite{chen2022learning}, and we  selected two common bias values $b$$=$$0.5$ and $b$$=$$0.9$ for validation. On the other hand, compared with SP-Motif-Mixed, GOOD-Motif studies \emph{covariate shift} and \emph{concept shift} by splitting the dataset by graph type and size, respectively.
    \item \textbf{GOOD-HIV} is a dataset consisting of a molecular graph where nodes are atoms and edges are chemical bonds. The label is whether the molecule can inhibit HIV replication. We split the dataset by molecular scaffold and size that should not determine the label, thereby injecting distribution shifts into the data.
    \item \textbf{OGBG-Molsider}, \textbf{OGBG-Molclintox}, and \textbf{OGBG-Molbace}~\cite{hu2020open} are three molecular property classification datasets provided by the OPEN GRAPH BENCHMARK (OGB). The datasets consist of molecules, and their  attributes are cast as binary labels. The \emph{scaffold splitting procedure} is used to split molecules according to their two-dimensional skeletons, which  introduce spurious correlations between functional groups due to the selection bias of the dataset.  
\end{itemize}
\subsubsection{Baselines}
We compare the proposed SNIGL method with three categories of methods, namely the state-of-the-art OOD methods from the Euclidean regime, and from the graph regime, as well as  the conventional GNN-based methods. The OOD methods from the graph regime include:
\begin{itemize}[leftmargin=*]
    \item \textbf{GroupDRO}~\cite{DBLP:journals/corr/abs-1911-08731} involves using regularization with Distributionally Robust Optimization (DRO) to enhance worst-group generalization in overparameterized neural networks.
    \item \textbf{DIR}~\cite{DBLP:conf/iclr/WuWZ0C22} enhances the interpretability and generalization of GNNs by identifying stable causal patterns via training distribution interventions and using classifiers on causal and non-causal parts for joint prediction.
    \item \textbf{CIGA}~\cite{chen2022learning} propose an information-theoretic objective to extract the invariant subgraphs that maximize the preservation of invariant intra-class information.
    \item \textbf{GIL}~\cite{li2022learning} proposes an invariant subgraph objective function for a GNN-based subgraph generator, leveraging variant subgraphs to infer potential environment labels.
    \item \textbf{OOD-GNN}~\cite{DBLP:journals/tkde/LiWZZ23} proposes to use random Fourier features and a global weight estimator to encourage the model to learn to eliminate statistical dependencies between relevant and irrelevant graph representations.
    \item \textbf{GSAT}~\cite{miao2022interpretable} injects stochasticity into attention weights to filter out task-irrelevant graph components and learns attention focused on task-relevant subgraphs, improving interpretability and prediction accuracy.
    \item \textbf{StableGNN}~\cite{fan2023generalizing} uses a differentiable graph pooling layer to extract subgraph-based representations and uses a causal variable differentiation regularizer to correct the biased training distribution 
    to remove spurious correlations.
    \item \textbf{GALA}~\cite{chen2024does} learns invariant subgraphs through proxy predictions from an auxiliary model that is sensitive to changes in the graph environment or distribution.
\end{itemize}
 The OOD methods from the Euclidean regime include:
\begin{itemize}[leftmargin=*]
    \item \textbf{DANN}~\cite{lempitsky2016domain} proposes a gradient reversal layer in the neural network architecture to develop features that are discriminative and domain-invariant for prediction tasks.
    \item \textbf{Coral}~\cite{sun2016deep} utilizes a nonlinear transformation to align the second-order statistics of activations between the source and target domains in neural networks for unsupervised domain adaptation.
    \item \textbf{Mixup}~\cite{DBLP:conf/iclr/ZhangCDL18} trains  neural networks (NNs) on convex combinations of pairs of examples and their labels, which regularizes the NN to favor simple linear behavior between training examples, thereby improving generalization.
    \item \textbf{IRM}~\cite{DBLP:journals/corr/abs-1907-02893}  proposes an invariant risk minimization strategy that forces the model’s decision boundary to be as consistent as possible across environments.
    \item \textbf{VREx}~\cite{krueger2021out} minimizes the model’s sensitivity to distribution shifts by reducing risk variance across training domains and uses optimization over extrapolated domains to achieve robustness to both causal and covariate shifts.
    \item \textbf{ERM}~\cite{vedantam2021empirical} involves using Fisher information, predictive entropy, and maximum mean discrepancy (MMD) to understand the OOD generalization of deep neural networks trained with empirical risk minimization.
\end{itemize}
Conventional GNN-based methods include:
\begin{itemize}[leftmargin=*]
    \item \textbf{GCN}~\cite{DBLP:conf/iclr/KipfW17} performs weighted aggregation of node features based on spectral graph theory to obtain node embeddings that can be used for downstream tasks.
    \item \textbf{GAT}~\cite{DBLP:journals/corr/abs-1710-10903} incorporates the self-attention mechanism into the message passing paradigm to adaptively select subgraphs that are discriminative for labels.
    \item \textbf{GraphSage}~\cite{hamilton2017inductive} generates node embeddings by sampling and aggregating features from the local neighborhoods of nodes, and a mini-batch training manner.
    \item \textbf{GIN}~\cite{DBLP:conf/iclr/XuHLJ19} performs node feature aggregation in an injective manner based on the theory of the 1-WL graph isomorphism test.
    \item \textbf{JKNet}~\cite{xu2018representation} involves flexibly leveraging different neighborhood ranges for each node to create structure-aware representations, enhancing performance by adapting to local neighborhood properties and tasks.
    \item \textbf{DIFFPOOL}~\cite{ying2018hierarchical} proposes a differentiable graph pooling module that generates hierarchical graph representations by learning soft cluster assignments for nodes. 
    \item \textbf{SGC}~\cite{wu2019simplifying} simplifies GCNs by removing non-linear activation functions and collapsing weight matrices, resulting in a scalable, interpretable linear model that acts as a low-pass filter followed by a linear classifier.
    \item \textbf{AttentiveFP}~\cite{xiong2019pushing} employs a multi-level, multi-stage attention mechanism to enhance molecular representation and interpretability to capture complex molecular interactions and relevant substructures for drug discovery tasks.
    \item \textbf{CMPNN}~\cite{10.5555/3491440.3491832} involves enhancing molecular embeddings by strengthening message interactions between nodes and edges through a communicative kernel and enriching the message generation process with a message booster module.
\end{itemize}
Since we mainly focus on domain generalization on graph data in this paper, 
we evaluate the effectiveness of the DG methods of graph data on all datasets, using larger datasets SP-Motif-Mixed, GOOD-Motif and GOOD-HIV to evaluate DG methods designed for general data, and smaller datasets OGBG-Molsider, OGBG-Molclintox and OGBG-Molbace to evaluate the conventional GNN-based methods.
\subsubsection{Implementation Details} 
The configurations of our SNIGL as well as  baselines are as follows.  
For a fair comparison, all methods utilize GIN as the underlying GNN backbone and use the max readout function to derive the embedding for the graph. We use ADAM optimizer in all experiments. We use ADAM optimizer in all experiments. All experiments are implemented by Pytorch on a single NVIDIA RTX A5000 24GB GPU. For our SNIGL, we set the dimensions of node embedding $d^c$ and $d^s$ in Eq.~\ref{equ:pcg} are both set to 300. We further set the regularization coefficient $\lambda$ in Eq.~\ref{equ:pns_objv2} to $0.001$. We used the Adam optimizer and the learning rate in the optimization algorithm in the training phase and the test phase was set as 0.001 and 0.0001, respectively. The maximum number of training epochs was set as 200. For the baselines, we tuned their settings empirically.

For performance evaluation, we closely follow the literature of GOOD~\cite{gui2022good} and OGBG~\cite{hu2020open}. Specifically, similar to the experimental setup in GOOD and OGBG, we report the ROC-AUC for all datasets, except for SP-Motif-Mixed where we use accuracy following CIGA~\cite{chen2022learning}. 
Further, we repeat the evaluation four times, select models based on the validation performances, and report the mean and standard deviation of the corresponding metric on the test set.
\subsection{Comparison to baselines}
\begin{table*}[htbp]
  \centering
  \caption{Performance on the graph classification task, measured in accuracy on SP-Motif-Mixed and ROC-AUC on other datasets. Standard deviation errors are given. The best performance is marked in bold, and the second best is underlined.}
  \scalebox{0.6}{
    \begin{tabular}{lrccrccrcc}
    \toprule
    Method &       & \multicolumn{2}{c}{SP-Motif-Mixed} &       & \multicolumn{2}{c}{GOOD-Motif	} &       & \multicolumn{2}{c}{GOOD-HIV} \\
    \midrule
          &       & bias=0.5 & bias=0.9 &       & motif & size  &       & scaffold & size \\
\cmidrule{3-4}\cmidrule{6-7}\cmidrule{9-10}
    GroupDRO &       & 0.5782(0.0172) & 0.5410(0.0997)     &       & 0.6196(0.0827) & 0.5169(0.0222) &       & 0.6815(0.0284) & 0.5775(0.0286) \\
    DIR   &       & 0.6382(0.0193) & 0.5508(0.6360)     &       & 0.3999(0.0550)  & 0.4483(0.0400) &       & 0.6844(0.0251) & 0.5767(0.0375) \\
    CIGAv1 &       & 0.5400(0.0173) & 0.5178(0.0729) &       & 0.6643(0.1131) & 0.4914(0.0834) &       & 0.6940(0.0239) & 0.6181(0.0168) \\
    CIGAv2 &       & 0.5930(0.0188) & 0.6341(0.0738) &       & 0.6715(0.0819) & 0.5442(0.0311) &       & 0.6940(0.0197)  & 0.5955(0.0256) \\
    GIL &       & 0.6947(0.0157) & 0.6741(0.0587) &       & 0.6274(0.0122) & 0.5147(0.0254) &       & 0.6925(0.0115)  & 0.5424(0.0274) \\
    OOD-GNN &       & 0.5784(0.00214) & 0.6184(0.0354) &       & 0.5714(0.0247) & 0.5874(0.0412) &       & 0.6854(0.0124)  & 0.5281(0.0341) \\
    GSAT  &       & 0.4223(0.0168)  & 0.5217(0.0311)     &       & 0.5513(0.0541) & \underline{0.6076}(0.0594) &       & 0.7007(0.0176) & 0.6073(0.0239) \\
    StableGNN &       & 0.6214(0.0214) & 0.6341(0.0165) &       & 0.5742(0.0214) & 0.5454(0.0142) &       & 0.6492(0.0261)  & 0.5269(0.0375) \\
    GALA  &       & \underline{0.7196}(0.0182)  & \underline{0.7040}(0.0763)     &       & 0.6041(0.015) & 0.5257(0.0082)  &       & 0.6864(0.0225) & 0.5948(0.0138) \\
    \midrule
    DANN  &       & 0.5793(0.0194) & 0.5110(0.9254)     &       & 0.5154(0.0728) & 0.5186(0.0244)  &       & 0.6943(0.0242) & \underline{0.6238}(0.0265) \\
    Coral &       & 0.5767(0.0191) & 0.5607(0.1807)     &       & 0.6623(0.0901) & 0.5371(0.0275) &       & \underline{0.7069}(0.0225) & 0.5939(0.0290) \\
    Mixup &       & 0.5153(0.0167) & 0.4533(0.0571)     &       & \underline{0.6967}(0.0586) & 0.5131(0.0256) &       & 0.7065(0.0186) & 0.5911(0.0311) \\
    IRM   &       & 0.5745(0.0186) & 0.4568(0.0488)  &       & 0.5993(0.1146) & 0.5368(0.0411) &       & 0.7017(0.0278) & 0.5994(0.0159) \\
    VREx  &       & 0.4737(0.0175) & 0.4886(0.0969) &       & 0.6653(0.0404) & 0.5447(0.0342) &       & 0.6934(0.0354) & 0.5849(0.0228) \\
    ERM   &       & 0.5725(0.0234)  & 0.4964(0.0463) &       & 0.6380(0.1036) & 0.5346(0.0408) &       & 0.6955(0.0239) & 0.5919(0.0229) \\
    \midrule
    SNIGL &       & \textbf{0.8000}(0.0302) & \textbf{0.7822}(0.0918) &       & \textbf{0.7748}(0.0221) & \textbf{0.6326}(0.0602) &       & \textbf{0.7278}(0.0072)  & \textbf{0.6341}(0.0031) \\
    \bottomrule
    \end{tabular}%
    }
  \label{tab:good}%
\end{table*}%
In this section, we answer Question Q1: how effective is our approach compared to existing methods? As shown in Table~\ref{tab:good}, we can find that our SNIGL method outperforms the other baselines with a large margin in different biases on the standard SPMotif-Mixed dataset, and in different split methods (i.e., structure, scaffold, size) on GOOD-Motif and GOOD-HIV datasets. 
In particular, we can obtain the following conclusions.
1) GALA achieves the second best performance on SP-Motif-Mixed under bias $b=0.5$ and $b=0.9$, while Mixup and GSAT achieve the second best performance on GOOD-Motif under motif-splitting and size-splitting, respectively, as well as Coral and DANN achieve the second best performance on GOOD-HIV under scaffold-splitting and size-splitting, respectively. 
Our proposed SNIGL is capable of achieving further improvements against GALA by $11.2\%$ and $11.1\%$ on SP-Motif-Mixed under bias $b=0.5$ and $b=0.9$, against Mixup and GSAT by $11.2\%$ and $4.1\%$ on GOOD-Motif, 
as well as against Coral and DANN by $3.0\%$ and $1.7\%$ on GOOD-HIV, indirectly reflecting that our method can extract the invariant subgraphs with the property of necessity and sufficiency. 
2) We also find that the performance drops with increasing biases, showing that over heavy bias can still influence generalization. 
3)  By comparing the variance of different methods, we can find that the variance of some baselines is large, this is because these methods generate the invariant subgraph by trading off two objects, which might lead to unstable results. In the meanwhile, the variance of our method is much smaller, reflecting the stability of our method.

\begin{wraptable}{l}{8.cm}\label{tab:ogb_cls2}
  \centering
  \caption{The ROC-AUC results on the graph classification task on the OGB datasets. Standard deviation errors are given.}
  \scalebox{0.65}{
    \begin{tabular}{lccc}
    \toprule
    Method         & Molsider & Molbace & Molclintox\\
    \midrule
    GroupDRO          & 0.5124(0.0145) & 0.7454(0.0245) & 0.7921(0.0245) \\
    DIR          & 0.5794(0.0111) & 0.7834(0.0145) & 0.8129(0.0307) \\
    CIGA          & 0.6745(0.0241) & 0.7546(0.0114) & 0.8841(0.0378) \\
    GIL          & 0.6350(0.0057) &   0.6915(0.0321)    &    0.8574(0.0164)  \\
    OOD-GNN        & 0.6400(0.0130) & \underline{0.8130}(0.0120) & 0.9140(0.0130)  \\
    GSAT            & 0.6654(0.0247) & 0.7845(0.0146) & 0.8974(0.0235)  \\
    StableGNN        & 0.5915(0.0117) & 0.7695(0.0327) & 0.8798(0.0237) \\
    GALA         & 0.5894(0.0051) & 0.7893(0.0037)  & 0.8737(0.0189) \\
     \midrule
    GCN           & 0.5843(0.0034) & 0.7689(0.0323) & 0.9027(0.0134)\\
    GAT           & 0.5956(0.0102) & 0.8124(0.0140) & 0.8798(0.0011)\\
    SAGE        & 0.6059(0.0016)  & 0.7425(0.0248) & 0.8877(0.0066)\\
    GIN           & 0.5817(0.0124) & 0.7638(0.0387) & \underline{0.9155}(0.0212)\\
    JKNet         & 0.5818(0.0159) & 0.7425(0.0291) & 0.8558(0.0217)\\
    DIFFPOOL       & 0.5758(0.0151)  & 0.7525(0.0116) & 0.8241(0.0167) \\
    SGC          & 0.5906(0.0032) & 0.6875(0.0021) & 0.8536(0.0028) \\
    AttentiveFP        & \underline{0.6919}(0.0148) & 0.7767(0.0026) & 0.8335(0.0216)  \\
    CMPNN        & 0.5799(0.0080) & 0.7215(0.0490) & 0.7947(0.0461) \\
    \midrule
    SNIGL        & \textbf{0.7278}(0.0120) & \textbf{0.8314}(0.0181) & \textbf{0.9323}(0.0042) \\
    \bottomrule
    \end{tabular}%
    }
  \label{tab:ogbg}%
\end{wraptable}
Next, we further compare the effectiveness of our SNIGL on more real-world datasets. The experiment results on the OGB datasets are shown in Table~\ref{tab:ogbg}. According to the experiment results, we can draw the following conclusions. 
1) The proposed SNIGL outperforms all other baselines on all the datasets, which is attributed to both the sufficiency and necessity restriction for invariant subgraphs and the ensemble training strategy with the help of the domain variant subgraphs. 

2) Some conventional GNN-based methods such as GCN and GIN do not achieve the ideal performance, reflecting that these methods have limited generalization. 
3) The causality-based baselines also achieve comparable performance and the methods based on environmental data augmentation achieve the closest results, reflecting the usefulness of the environment augmentation. However, since it is difficult for these methods to extract necessary and sufficient invariant subgraphs, the experimental results of these methods, such as StableGNN on Mollintox, DIR on Molsider and Mollintox, are difficult to achieve ideal results.

\subsection{Ablation Study}
As Section~\ref{sec:alg} states, the PNS risk and the ensemble strategy of combining domain-variant and domain-invariant features are key components in our proposed SNIGL. To answer questions Q2 and Q3 to show if these key components benefit the generalization performance of our SNIGL, we can thus derive the following variants of SNIGL:
(1) SNIGL without PNS risk denoted as SNIGLw/oPNS; (2) SNIGL without ensemble strategy for the combination of variant and invariant features denoted as SNIGLw/oEN. To show the impact of PNS risk and the ensemble strategy, we compare SNIGL with the above variants.
\begin{wrapfigure}{l}{7.0cm} 
    \centering
\includegraphics[width=0.5\columnwidth]{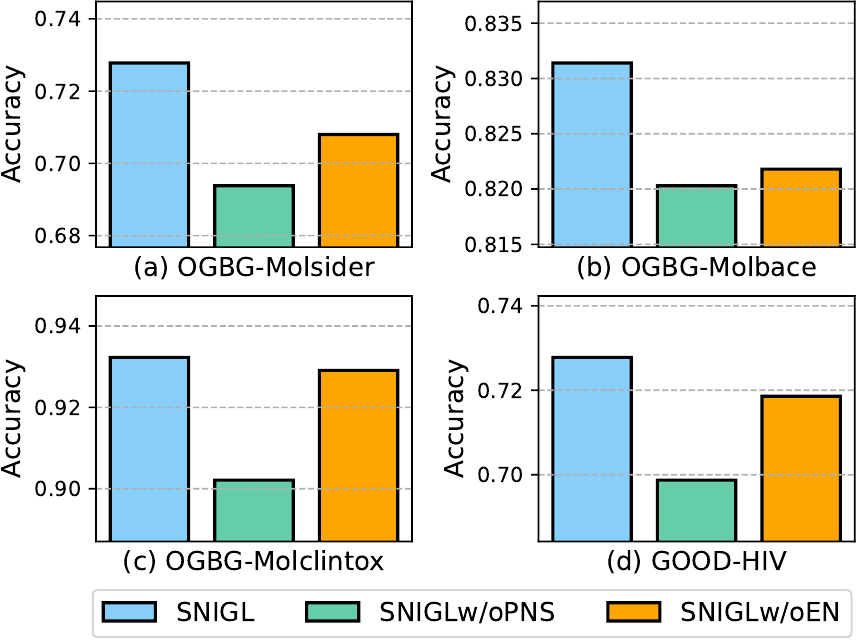}
\vspace{-2mm}
    \caption{Ablation studies on four datasets. We explore the impact of the key components of our SNIGL, i.e, PNS risk and the ensemble strategy }
\vspace{-5mm}
    \label{fig:ablation}
\end{wrapfigure}
In Figure~\ref{fig:ablation}, we observe that SNIGL achieves better performance than the variants in terms of accuracy, demonstrating the effectiveness of the PNS risk and our ensemble strategy. Firstly, in order to demonstrate the impact of the PNS risk on SNIGL's performance, we compare SNIGL with the variant SNIGLw/oPNS. As Figure~\ref{fig:ablation}' blue/green columns show, we observe that the accuracy of SNIGLw/oPNS is lower than that of SNIGL on all four datasets, demonstrating the importance of PNS risk in GNNs, that is, the invariant features learned from these data sets contain many insufficient or unnecessary invariant features, which may be harmful to prediction. Secondly, in order to investigate the impact of the ensemble strategy, 
we removed the combination of the trained invariant GNN $g_{\phi^c} \circ f_{\theta^c}$ and the unstable GNN $g_{\phi^{s,e}} \circ f_{\theta^s}$ and just used $g_{\phi^c} \circ f_{\theta^c}$ to predict,  and the comparison between SNIGL and SNIGLw/oEN  is as Figure~\ref{fig:ablation} illustrates. As can be seen, SNIGL significantly outperforms SNIGLw/oEN on Molsider and Molbace datasets. 
This result reflects that Molsider and Molbace may contain few sufficient and necessary invariant subgraphs, so the prediction performance drops significantly without integrating domain variant features to assist prediction. On the contrary, the performance gap between  SNIGL and SNIGLw/oEN is more subtle on OGBG-Molclintox and GOOD-HIV datasets, which is intuitive since these two datasets may contain more necessary and sufficient invariant features, so the performance will remain stable with or without the assistance of domain variant features.

\subsection{Visualization}
\begin{figure}
    \centering
\includegraphics[width=0.95\columnwidth]{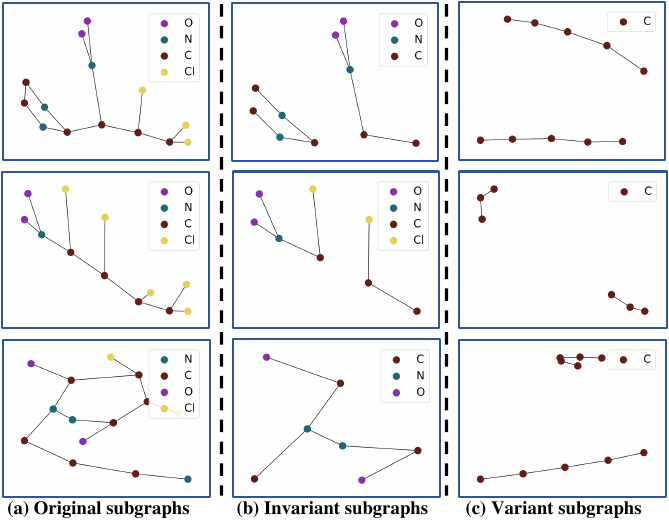}
\vspace{-3mm}
    \caption{Visualization of molecule examples in the OGBG-HIV dataset. Nodes with different colors denote different atoms, and edges denote different chemical bonds.}
    \vspace{-3mm}
    \label{fig:vis}
\end{figure}
To answer question Q4, i.e., what patterns are actually learned by our SNIGL and how such patterns improve the generalization of SNIGL, we visualize the domain invariant and domain variant subgraphs learned by SNIGL on the OGBG-HIV dataset as shown in Figure~\ref{fig:vis}. The visualization results allow us to draw the following conclusions. 
1) The domain invariant substructures are sparse  basic structures, showing that our method is capable of identifying necessary and sufficient latent substructures. 
2) Our SNIGL model is capable of generating reasonable molecular substructures comprising basic functional groups. For example, SNIGL can extract substructures such as ``-NO2'' in the second line's domain invariant subgraph, which is composed of two purple nodes and one green node. 
This may offer an efficient avenue for uncovering the latent value of molecules.
3) We find that some atoms appear neither in the domain-invariant subgraph nor in the domain-variant subgraph. For example, in the example on the third line, atom $-Cl$ disappears in both the domain invariant/variant subgraphs, which indicates that our model can indeed remove redundant structures.

\section{Related Work}
\subsection{Graph Out-of-Distrubtion. }
In this subsection, we provide an introduction to domain generalization of graph classification \cite{fan2023generalizing,yang2022learning,guo2020graseq,liu2022graph,chen2022learning,10027780}. Existing works on out-of-distribution (OOD) \cite{DBLP:journals/corr/abs-2108-13624} mainly focus on the fields of computer vision \cite{zhang2021deep,zhang2022multi} and natural language processing \cite{chen2021hiddencut}, but the OOD challenge on graph-structured data receives less attention. Considering that the existing GNNs lack out-of-distribution generalization \cite{li2022graphde,zhao2020uncertainty,liu2023flood,sui2022causal,sun2022does,DBLP:journals/tkde/LiWZZ23,DBLP:journals/corr/abs-2311-04837}, Li et al. \cite{DBLP:journals/tkde/LiWZZ23} proposed OOD-GNN to tackle the graph OOD (OOD) challenge by addressing the statistical dependence between relevant and irrelevant graph representations. Recognizing that spurious correlations often undermine the generalization of graph neural networks (GNN), Fan et al. propose the StableGNN \cite{fan2023generalizing}, which extracts causal representation for GNNs with the help of stable learning. Aiming to mitigate the selection bias behind graph-structured data, Wu et al. further proposes the DIR model \cite{DBLP:conf/iclr/WuWZ0C22} to mine the invariant causal rationales via causal intervention. These methods essentially employ causal effect estimation to make invariant and spurious subgraphs independent. And the augmentation-based model is another type of the important method. Liu et al. \cite{10.1145/3534678.3539347} employ augmentation to improve the robustness and decompose the observed graph into the environment part and the rationale part. Recently, Chen et al.~\cite{chen2022learning, chen2024does} investigate the usefulness of
 augmented environment information from a theoretical perspective. And Li et. al \cite{DBLP:journals/corr/abs-2311-04837} further consider a concrete scenario of graph OOD, i.e., molecular property prediction from the perspective of latent variables identification \cite{DBLP:conf/nips/0001C0SH023}. Although the aforementioned methods mitigate the distribution shift of graph data to some extent, they can not extract the invariant subgraphs with Necessity and Sufficiency \cite{DBLP:conf/nips/YangZF0LTWW23}. Moreover, as \cite{DBLP:conf/nips/Eastwood0NPKS23} discussed, the domain variant subgraphs also play a critical role when the data with noisy labels \cite{liu2015classification,wu2024time,bai2023subclass}. In this paper, we propose the SNIGL method, which unifies the extraction of the invariant latent subgraph with necessity and sufficiency and the exploitation of variant subgraphs via an ensemble manner. 

\subsection{Probability of Necessity and Sufficiency}

As the probability of causation, the Probability of Necessity and Sufficiency (PNS) can be used to measure the ``if and only if'' of the relationship between two events. Additionally,  the Probability of Necessity (PN) and Probability of  Sufficiency (PS) are used to evaluate the ``sufficiency cause'' and ``necessity cause'', respectively. 
Pearl \cite{pearl2000models} and Tian and Pearl \cite{tian2000probabilities} formulated precise meanings for the probabilities of causation using structural causal models. 
The issue of the identifiability of PNS initially attracted widespread attention 
\cite{galles1998axiomatic,halpern2000axiomatizing,pearl2009causality,DBLP:conf/nips/CaiWJ023,tian2000probabilities, li2019unit,li2022unit, DBLP:conf/ijcai/MuellerLP22, dawid2017probability, li2024probabilities, gleiss2019quantifying, zhang2022partial, li2022bounds}.
Kuroki and Cai~\cite{kuroki2011statistical} and Tian and Pearl~\cite{tian2000probabilities} demonstrated how to bound these quantities from data obtained in experimental and observational studies to solve this problem.
These bounds lie within the range in which the probability of causation must lie, however, it has been pointed out that these bounds are too wide to assess the probability of causation.
To overcome this difficulty, 
Pearl demonstrated that identifying the probabilities of causation requires specific functional relationships between the causes and their outcomes \cite{pearl2000models}. 
Recently, incorporating PNS into various application scenarios has also attracted much attention and currently has many applications \cite{tan2022learning, galhotra2021explaining,watson2021local,DBLP:journals/corr/abs-2212-07056,DBLP:conf/ijcai/MuellerLP22,beckers2021causal,shingaki2021identification}. 
For example, in ML explainability, CF$^2$ \cite{tan2022learning}, LEWIS \cite{galhotra2021explaining}, LENS \cite{watson2021local}, NSEG \cite{DBLP:journals/corr/abs-2212-07056} and FANS~\cite{DBLP:journals/corr/abs-2402-08845} use sufficiency or necessity to measure the contribution of input feature subsets to the model's predictions. In the causal effect estimation problem \cite{DBLP:conf/ijcai/MuellerLP22,beckers2021causal,shingaki2021identification}, it can be used to learn individual responses from population data \cite{DBLP:conf/ijcai/MuellerLP22}. 
In the out-of-distribution generalization problem, CaSN employs PNS to extract domain-invariant information \cite{DBLP:conf/nips/YangZF0LTWW23}. 
Although CaSN is effective in extracting sufficient and necessary invariant representations, CaSN ignores the fact that the data may not have sufficient and necessary invariant features. Furthermore, CaSN requires the assumption that the classifier is linear in order to learn sufficient and necessary invariant features, which is unrealistic for GNN-based predictors. 

\section{Conclusion}
This paper presents a unified framework called  Sufficiency and Necessity Inspired Graph Learning (SNIGL) for graph out-of-distribution learning, leveraging the probability of necessity and sufficiency (PNS) for invariant subgraph learning, and combining domain-variant subgraphs with learned invariant subgraphs for ensemble reasoning. Initially, we outline a conventional graph generation process. We then propose that sufficient and necessary invariant subgraphs can be identified by maximizing the PNS. To address the computational challenges associated with PNS, we introduce a flexible lower bound for PNS under mild conditions. Additionally, we propose a strategy to ensemble variant and invariant features on the test set, utilizing variant features to mitigate the sparsity issue of necessary and sufficient invariant features. Our SNIGL  demonstrates superior performance compared to state-of-the-art methods. However, a key limitation of our approach we need to assume that variant and invariant features are conditionally independent in using domain variant features to address the sparsity issue of sufficient and necessary invariant features. 
Future work should focus on developing a more flexible method to overcome this limitation.
\clearpage

\bibliographystyle{ACM-Reference-Format}
\bibliography{citation}


\begin{thebibliography}{77}


\ifx \showCODEN    \undefined \def \showCODEN     #1{\unskip}     \fi
\ifx \showDOI      \undefined \def \showDOI       #1{#1}\fi
\ifx \showISBNx    \undefined \def \showISBNx     #1{\unskip}     \fi
\ifx \showISBNxiii \undefined \def \showISBNxiii  #1{\unskip}     \fi
\ifx \showISSN     \undefined \def \showISSN      #1{\unskip}     \fi
\ifx \showLCCN     \undefined \def \showLCCN      #1{\unskip}     \fi
\ifx \shownote     \undefined \def \shownote      #1{#1}          \fi
\ifx \showarticletitle \undefined \def \showarticletitle #1{#1}   \fi
\ifx \showURL      \undefined \def \showURL       {\relax}        \fi
\providecommand\bibfield[2]{#2}
\providecommand\bibinfo[2]{#2}
\providecommand\natexlab[1]{#1}
\providecommand\showeprint[2][]{arXiv:#2}

\bibitem[Arjovsky et~al\mbox{.}(2019)]%
        {DBLP:journals/corr/abs-1907-02893}
\bibfield{author}{\bibinfo{person}{Mart{\'{\i}}n Arjovsky},
  \bibinfo{person}{L{\'{e}}on Bottou}, \bibinfo{person}{Ishaan Gulrajani},
  {and} \bibinfo{person}{David Lopez{-}Paz}.} \bibinfo{year}{2019}\natexlab{}.
\newblock \showarticletitle{Invariant Risk Minimization}.
\newblock \bibinfo{journal}{\emph{CoRR}}  \bibinfo{volume}{abs/1907.02893}
  (\bibinfo{year}{2019}).
\newblock
\showeprint[arXiv]{1907.02893}
\urldef\tempurl%
\url{http://arxiv.org/abs/1907.02893}
\showURL{%
\tempurl}


\bibitem[Bai et~al\mbox{.}(2023)]%
        {bai2023subclass}
\bibfield{author}{\bibinfo{person}{Yingbin Bai}, \bibinfo{person}{Zhongyi Han},
  \bibinfo{person}{Erkun Yang}, \bibinfo{person}{Jun Yu}, \bibinfo{person}{Bo
  Han}, \bibinfo{person}{Dadong Wang}, {and} \bibinfo{person}{Tongliang Liu}.}
  \bibinfo{year}{2023}\natexlab{}.
\newblock \showarticletitle{Subclass-Dominant Label Noise: A Counterexample for
  the Success of Early Stopping}. In \bibinfo{booktitle}{\emph{Thirty-seventh
  Conference on Neural Information Processing Systems}}.
\newblock


\bibitem[Beckers(2021)]%
        {beckers2021causal}
\bibfield{author}{\bibinfo{person}{Sander Beckers}.}
  \bibinfo{year}{2021}\natexlab{}.
\newblock \showarticletitle{Causal sufficiency and actual causation}.
\newblock \bibinfo{journal}{\emph{Journal of Philosophical Logic}}
  \bibinfo{volume}{50}, \bibinfo{number}{6} (\bibinfo{year}{2021}),
  \bibinfo{pages}{1341--1374}.
\newblock


\bibitem[Cai et~al\mbox{.}(2023)]%
        {DBLP:conf/nips/CaiWJ023}
\bibfield{author}{\bibinfo{person}{Hengrui Cai}, \bibinfo{person}{Yixin Wang},
  \bibinfo{person}{Michael~I. Jordan}, {and} \bibinfo{person}{Rui Song}.}
  \bibinfo{year}{2023}\natexlab{}.
\newblock \showarticletitle{On Learning Necessary and Sufficient Causal
  Graphs}. In \bibinfo{booktitle}{\emph{Advances in Neural Information
  Processing Systems 36: Annual Conference on Neural Information Processing
  Systems 2023, NeurIPS 2023, New Orleans, LA, USA, December 10 - 16, 2023}},
  \bibfield{editor}{\bibinfo{person}{Alice Oh}, \bibinfo{person}{Tristan
  Naumann}, \bibinfo{person}{Amir Globerson}, \bibinfo{person}{Kate Saenko},
  \bibinfo{person}{Moritz Hardt}, {and} \bibinfo{person}{Sergey Levine}}
  (Eds.).
\newblock


\bibitem[Cai et~al\mbox{.}(2022)]%
        {DBLP:journals/corr/abs-2212-07056}
\bibfield{author}{\bibinfo{person}{Ruichu Cai}, \bibinfo{person}{Yuxuan Zhu},
  \bibinfo{person}{Xuexin Chen}, \bibinfo{person}{Yuan Fang},
  \bibinfo{person}{Min Wu}, \bibinfo{person}{Jie Qiao}, {and}
  \bibinfo{person}{Zhifeng Hao}.} \bibinfo{year}{2022}\natexlab{}.
\newblock \showarticletitle{On the Probability of Necessity and Sufficiency of
  Explaining Graph Neural Networks: {A} Lower Bound Optimization Approach}.
\newblock \bibinfo{journal}{\emph{CoRR}}  \bibinfo{volume}{abs/2212.07056}
  (\bibinfo{year}{2022}).
\newblock
\urldef\tempurl%
\url{https://doi.org/10.48550/ARXIV.2212.07056}
\showDOI{\tempurl}
\showeprint[arXiv]{2212.07056}


\bibitem[Chen et~al\mbox{.}(2021b)]%
        {chen2021hiddencut}
\bibfield{author}{\bibinfo{person}{Jiaao Chen}, \bibinfo{person}{Dinghan Shen},
  \bibinfo{person}{Weizhu Chen}, {and} \bibinfo{person}{Diyi Yang}.}
  \bibinfo{year}{2021}\natexlab{b}.
\newblock \showarticletitle{Hiddencut: Simple data augmentation for natural
  language understanding with better generalizability}. In
  \bibinfo{booktitle}{\emph{Proceedings of the 59th Annual Meeting of the
  Association for Computational Linguistics and the 11th International Joint
  Conference on Natural Language Processing (Volume 1: Long Papers)}}.
  \bibinfo{pages}{4380--4390}.
\newblock


\bibitem[Chen et~al\mbox{.}(2021a)]%
        {CHEN2021100595}
\bibfield{author}{\bibinfo{person}{Lihua Chen}, \bibinfo{person}{Ghanshyam
  Pilania}, \bibinfo{person}{Rohit Batra}, \bibinfo{person}{Tran~Doan Huan},
  \bibinfo{person}{Chiho Kim}, \bibinfo{person}{Christopher Kuenneth}, {and}
  \bibinfo{person}{Rampi Ramprasad}.} \bibinfo{year}{2021}\natexlab{a}.
\newblock \showarticletitle{Polymer informatics: Current status and critical
  next steps}.
\newblock \bibinfo{journal}{\emph{Materials Science and Engineering: R:
  Reports}}  \bibinfo{volume}{144} (\bibinfo{year}{2021}),
  \bibinfo{pages}{100595}.
\newblock
\showISSN{0927-796X}
\urldef\tempurl%
\url{https://doi.org/10.1016/j.mser.2020.100595}
\showDOI{\tempurl}


\bibitem[Chen et~al\mbox{.}(2024b)]%
        {DBLP:journals/corr/abs-2402-08845}
\bibfield{author}{\bibinfo{person}{Xuexin Chen}, \bibinfo{person}{Ruichu Cai},
  \bibinfo{person}{Zhengting Huang}, \bibinfo{person}{Yuxuan Zhu},
  \bibinfo{person}{Julien Horwood}, \bibinfo{person}{Zhifeng Hao},
  \bibinfo{person}{Zijian Li}, {and} \bibinfo{person}{Jos{\'{e}}~Miguel
  Hern{\'{a}}ndez{-}Lobato}.} \bibinfo{year}{2024}\natexlab{b}.
\newblock \showarticletitle{Feature Attribution with Necessity and Sufficiency
  via Dual-stage Perturbation Test for Causal Explanation}.
\newblock \bibinfo{journal}{\emph{CoRR}}  \bibinfo{volume}{abs/2402.08845}
  (\bibinfo{year}{2024}).
\newblock
\showeprint[arXiv]{2402.08845}


\bibitem[Chen et~al\mbox{.}(2024a)]%
        {chen2024does}
\bibfield{author}{\bibinfo{person}{Yongqiang Chen}, \bibinfo{person}{Yatao
  Bian}, \bibinfo{person}{Kaiwen Zhou}, \bibinfo{person}{Binghui Xie},
  \bibinfo{person}{Bo Han}, {and} \bibinfo{person}{James Cheng}.}
  \bibinfo{year}{2024}\natexlab{a}.
\newblock \showarticletitle{Does invariant graph learning via environment
  augmentation learn invariance?}
\newblock \bibinfo{journal}{\emph{Advances in Neural Information Processing
  Systems}}  \bibinfo{volume}{36} (\bibinfo{year}{2024}).
\newblock


\bibitem[Chen et~al\mbox{.}(2022a)]%
        {chen2022invariance}
\bibfield{author}{\bibinfo{person}{Yongqiang Chen}, \bibinfo{person}{Yonggang
  Zhang}, \bibinfo{person}{Yatao Bian}, \bibinfo{person}{Han Yang},
  \bibinfo{person}{MA KAILI}, \bibinfo{person}{Binghui Xie},
  \bibinfo{person}{Tongliang Liu}, \bibinfo{person}{Bo Han}, {and}
  \bibinfo{person}{James Cheng}.} \bibinfo{year}{2022}\natexlab{a}.
\newblock \showarticletitle{Invariance principle meets out-of-distribution
  generalization on graphs}. In \bibinfo{booktitle}{\emph{ICML 2022: Workshop
  on Spurious Correlations, Invariance and Stability}}.
\newblock


\bibitem[Chen et~al\mbox{.}(2022b)]%
        {chen2022learning}
\bibfield{author}{\bibinfo{person}{Yongqiang Chen}, \bibinfo{person}{Yonggang
  Zhang}, \bibinfo{person}{Yatao Bian}, \bibinfo{person}{Han Yang},
  \bibinfo{person}{MA Kaili}, \bibinfo{person}{Binghui Xie},
  \bibinfo{person}{Tongliang Liu}, \bibinfo{person}{Bo Han}, {and}
  \bibinfo{person}{James Cheng}.} \bibinfo{year}{2022}\natexlab{b}.
\newblock \showarticletitle{Learning causally invariant representations for
  out-of-distribution generalization on graphs}.
\newblock \bibinfo{journal}{\emph{Advances in Neural Information Processing
  Systems}}  \bibinfo{volume}{35} (\bibinfo{year}{2022}),
  \bibinfo{pages}{22131--22148}.
\newblock


\bibitem[Dawid et~al\mbox{.}(2017)]%
        {dawid2017probability}
\bibfield{author}{\bibinfo{person}{A~Philip Dawid}, \bibinfo{person}{Monica
  Musio}, {and} \bibinfo{person}{Rossella Murtas}.}
  \bibinfo{year}{2017}\natexlab{}.
\newblock \showarticletitle{The probability of causation}.
\newblock \bibinfo{journal}{\emph{Law, Probability and Risk}}
  \bibinfo{volume}{16}, \bibinfo{number}{4} (\bibinfo{year}{2017}),
  \bibinfo{pages}{163--179}.
\newblock


\bibitem[Eastwood et~al\mbox{.}(2023)]%
        {DBLP:conf/nips/Eastwood0NPKS23}
\bibfield{author}{\bibinfo{person}{Cian Eastwood}, \bibinfo{person}{Shashank
  Singh}, \bibinfo{person}{Andrei~Liviu Nicolicioiu},
  \bibinfo{person}{Marin~Vlastelica Pogancic}, \bibinfo{person}{Julius von
  K{\"{u}}gelgen}, {and} \bibinfo{person}{Bernhard Sch{\"{o}}lkopf}.}
  \bibinfo{year}{2023}\natexlab{}.
\newblock \showarticletitle{Spuriosity Didn't Kill the Classifier: Using
  Invariant Predictions to Harness Spurious Features}. In
  \bibinfo{booktitle}{\emph{Advances in Neural Information Processing Systems
  36: Annual Conference on Neural Information Processing Systems 2023, NeurIPS
  2023, New Orleans, LA, USA, December 10 - 16, 2023}},
  \bibfield{editor}{\bibinfo{person}{Alice Oh}, \bibinfo{person}{Tristan
  Naumann}, \bibinfo{person}{Amir Globerson}, \bibinfo{person}{Kate Saenko},
  \bibinfo{person}{Moritz Hardt}, {and} \bibinfo{person}{Sergey Levine}}
  (Eds.).
\newblock


\bibitem[Fan et~al\mbox{.}(2023)]%
        {fan2023generalizing}
\bibfield{author}{\bibinfo{person}{Shaohua Fan}, \bibinfo{person}{Xiao Wang},
  \bibinfo{person}{Chuan Shi}, \bibinfo{person}{Peng Cui}, {and}
  \bibinfo{person}{Bai Wang}.} \bibinfo{year}{2023}\natexlab{}.
\newblock \showarticletitle{Generalizing graph neural networks on
  out-of-distribution graphs}.
\newblock \bibinfo{journal}{\emph{IEEE Transactions on Pattern Analysis and
  Machine Intelligence}} (\bibinfo{year}{2023}).
\newblock


\bibitem[Fang et~al\mbox{.}(2022)]%
        {10027780}
\bibfield{author}{\bibinfo{person}{Zheng Fang}, \bibinfo{person}{Ziyun Zhang},
  \bibinfo{person}{Guojie Song}, \bibinfo{person}{Yingxue Zhang},
  \bibinfo{person}{Dong Li}, \bibinfo{person}{Jianye Hao}, {and}
  \bibinfo{person}{Xi Wang}.} \bibinfo{year}{2022}\natexlab{}.
\newblock \showarticletitle{Invariant Factor Graph Neural Networks}. In
  \bibinfo{booktitle}{\emph{2022 IEEE International Conference on Data Mining
  (ICDM)}}. \bibinfo{pages}{933--938}.
\newblock
\urldef\tempurl%
\url{https://doi.org/10.1109/ICDM54844.2022.00110}
\showDOI{\tempurl}


\bibitem[Galhotra et~al\mbox{.}(2021)]%
        {galhotra2021explaining}
\bibfield{author}{\bibinfo{person}{Sainyam Galhotra}, \bibinfo{person}{Romila
  Pradhan}, {and} \bibinfo{person}{Babak Salimi}.}
  \bibinfo{year}{2021}\natexlab{}.
\newblock \showarticletitle{Explaining black-box algorithms using probabilistic
  contrastive counterfactuals}. In \bibinfo{booktitle}{\emph{Proceedings of the
  2021 International Conference on Management of Data}}.
  \bibinfo{pages}{577--590}.
\newblock


\bibitem[Galles and Pearl(1998)]%
        {galles1998axiomatic}
\bibfield{author}{\bibinfo{person}{David Galles} {and} \bibinfo{person}{Judea
  Pearl}.} \bibinfo{year}{1998}\natexlab{}.
\newblock \showarticletitle{An axiomatic characterization of causal
  counterfactuals}.
\newblock \bibinfo{journal}{\emph{Foundations of Science}}  \bibinfo{volume}{3}
  (\bibinfo{year}{1998}), \bibinfo{pages}{151--182}.
\newblock


\bibitem[Gleiss and Schemper(2019)]%
        {gleiss2019quantifying}
\bibfield{author}{\bibinfo{person}{Andreas Gleiss} {and}
  \bibinfo{person}{Michael Schemper}.} \bibinfo{year}{2019}\natexlab{}.
\newblock \showarticletitle{Quantifying degrees of necessity and of sufficiency
  in cause-effect relationships with dichotomous and survival outcomes}.
\newblock \bibinfo{journal}{\emph{Statistics in Medicine}}
  \bibinfo{volume}{38}, \bibinfo{number}{23} (\bibinfo{year}{2019}),
  \bibinfo{pages}{4733--4748}.
\newblock


\bibitem[Gui et~al\mbox{.}(2022)]%
        {gui2022good}
\bibfield{author}{\bibinfo{person}{Shurui Gui}, \bibinfo{person}{Xiner Li},
  \bibinfo{person}{Limei Wang}, {and} \bibinfo{person}{Shuiwang Ji}.}
  \bibinfo{year}{2022}\natexlab{}.
\newblock \showarticletitle{Good: A graph out-of-distribution benchmark}.
\newblock \bibinfo{journal}{\emph{Advances in Neural Information Processing
  Systems}}  \bibinfo{volume}{35} (\bibinfo{year}{2022}),
  \bibinfo{pages}{2059--2073}.
\newblock


\bibitem[Guo et~al\mbox{.}(2020)]%
        {guo2020graseq}
\bibfield{author}{\bibinfo{person}{Zhichun Guo}, \bibinfo{person}{Wenhao Yu},
  \bibinfo{person}{Chuxu Zhang}, \bibinfo{person}{Meng Jiang}, {and}
  \bibinfo{person}{Nitesh~V Chawla}.} \bibinfo{year}{2020}\natexlab{}.
\newblock \showarticletitle{GraSeq: graph and sequence fusion learning for
  molecular property prediction}. In \bibinfo{booktitle}{\emph{Proceedings of
  the 29th ACM international conference on information \& knowledge
  management}}. \bibinfo{pages}{435--443}.
\newblock


\bibitem[Halpern(2000)]%
        {halpern2000axiomatizing}
\bibfield{author}{\bibinfo{person}{Joseph~Y Halpern}.}
  \bibinfo{year}{2000}\natexlab{}.
\newblock \showarticletitle{Axiomatizing causal reasoning}.
\newblock \bibinfo{journal}{\emph{Journal of Artificial Intelligence Research}}
   \bibinfo{volume}{12} (\bibinfo{year}{2000}), \bibinfo{pages}{317--337}.
\newblock


\bibitem[Hamilton et~al\mbox{.}(2017)]%
        {hamilton2017inductive}
\bibfield{author}{\bibinfo{person}{Will Hamilton}, \bibinfo{person}{Zhitao
  Ying}, {and} \bibinfo{person}{Jure Leskovec}.}
  \bibinfo{year}{2017}\natexlab{}.
\newblock \showarticletitle{Inductive representation learning on large graphs}.
\newblock \bibinfo{journal}{\emph{Advances in neural information processing
  systems}}  \bibinfo{volume}{30} (\bibinfo{year}{2017}).
\newblock


\bibitem[Hern{\'a}n and Robins(2010)]%
        {hernan2010causal}
\bibfield{author}{\bibinfo{person}{Miguel~A Hern{\'a}n} {and}
  \bibinfo{person}{James~M Robins}.} \bibinfo{year}{2010}\natexlab{}.
\newblock \bibinfo{title}{Causal inference}.
\newblock
\newblock


\bibitem[Hu et~al\mbox{.}(2020)]%
        {hu2020open}
\bibfield{author}{\bibinfo{person}{Weihua Hu}, \bibinfo{person}{Matthias Fey},
  \bibinfo{person}{Marinka Zitnik}, \bibinfo{person}{Yuxiao Dong},
  \bibinfo{person}{Hongyu Ren}, \bibinfo{person}{Bowen Liu},
  \bibinfo{person}{Michele Catasta}, {and} \bibinfo{person}{Jure Leskovec}.}
  \bibinfo{year}{2020}\natexlab{}.
\newblock \showarticletitle{Open graph benchmark: Datasets for machine learning
  on graphs}.
\newblock \bibinfo{journal}{\emph{Advances in neural information processing
  systems}}  \bibinfo{volume}{33} (\bibinfo{year}{2020}),
  \bibinfo{pages}{22118--22133}.
\newblock


\bibitem[Jang et~al\mbox{.}(2017)]%
        {DBLP:conf/iclr/JangGP17}
\bibfield{author}{\bibinfo{person}{Eric Jang}, \bibinfo{person}{Shixiang Gu},
  {and} \bibinfo{person}{Ben Poole}.} \bibinfo{year}{2017}\natexlab{}.
\newblock \showarticletitle{Categorical Reparameterization with
  Gumbel-Softmax}. In \bibinfo{booktitle}{\emph{5th International Conference on
  Learning Representations, {ICLR} 2017, Toulon, France, April 24-26, 2017,
  Conference Track Proceedings}}. \bibinfo{publisher}{OpenReview.net}.
\newblock
\urldef\tempurl%
\url{https://openreview.net/forum?id=rkE3y85ee}
\showURL{%
\tempurl}


\bibitem[Kipf and Welling(2017)]%
        {DBLP:conf/iclr/KipfW17}
\bibfield{author}{\bibinfo{person}{Thomas~N. Kipf} {and} \bibinfo{person}{Max
  Welling}.} \bibinfo{year}{2017}\natexlab{}.
\newblock \showarticletitle{Semi-Supervised Classification with Graph
  Convolutional Networks}. In \bibinfo{booktitle}{\emph{5th International
  Conference on Learning Representations, {ICLR} 2017, Toulon, France, April
  24-26, 2017, Conference Track Proceedings}}.
  \bibinfo{publisher}{OpenReview.net}.
\newblock
\urldef\tempurl%
\url{https://openreview.net/forum?id=SJU4ayYgl}
\showURL{%
\tempurl}


\bibitem[Krueger et~al\mbox{.}(2021)]%
        {krueger2021out}
\bibfield{author}{\bibinfo{person}{David Krueger}, \bibinfo{person}{Ethan
  Caballero}, \bibinfo{person}{Joern-Henrik Jacobsen}, \bibinfo{person}{Amy
  Zhang}, \bibinfo{person}{Jonathan Binas}, \bibinfo{person}{Dinghuai Zhang},
  \bibinfo{person}{Remi Le~Priol}, {and} \bibinfo{person}{Aaron Courville}.}
  \bibinfo{year}{2021}\natexlab{}.
\newblock \showarticletitle{Out-of-distribution generalization via risk
  extrapolation (rex)}. In \bibinfo{booktitle}{\emph{International Conference
  on Machine Learning}}. PMLR, \bibinfo{pages}{5815--5826}.
\newblock


\bibitem[Kuroki and Cai(2011)]%
        {kuroki2011statistical}
\bibfield{author}{\bibinfo{person}{Manabu Kuroki} {and}
  \bibinfo{person}{Zhihong Cai}.} \bibinfo{year}{2011}\natexlab{}.
\newblock \showarticletitle{Statistical Analysis of ‘Probabilities of
  Causation’Using Co-variate Information}.
\newblock \bibinfo{journal}{\emph{Scandinavian Journal of Statistics}}
  \bibinfo{volume}{38}, \bibinfo{number}{3} (\bibinfo{year}{2011}),
  \bibinfo{pages}{564--577}.
\newblock


\bibitem[Lempitsky(2016)]%
        {lempitsky2016domain}
\bibfield{author}{\bibinfo{person}{Victor Lempitsky}.}
  \bibinfo{year}{2016}\natexlab{}.
\newblock \showarticletitle{Domain-adversarial training of neural networks}.
\newblock \bibinfo{journal}{\emph{The Journal}} (\bibinfo{year}{2016}).
\newblock


\bibitem[Li and Pearl(2019)]%
        {li2019unit}
\bibfield{author}{\bibinfo{person}{Ang Li} {and} \bibinfo{person}{Judea
  Pearl}.} \bibinfo{year}{2019}\natexlab{}.
\newblock \showarticletitle{Unit selection based on counterfactual logic}. In
  \bibinfo{booktitle}{\emph{Proceedings of the Twenty-Eighth International
  Joint Conference on Artificial Intelligence}}.
\newblock


\bibitem[Li and Pearl(2022a)]%
        {li2022bounds}
\bibfield{author}{\bibinfo{person}{Ang Li} {and} \bibinfo{person}{Judea
  Pearl}.} \bibinfo{year}{2022}\natexlab{a}.
\newblock \showarticletitle{Bounds on causal effects and application to high
  dimensional data}. In \bibinfo{booktitle}{\emph{Proceedings of the AAAI
  Conference on Artificial Intelligence}}, Vol.~\bibinfo{volume}{36}.
  \bibinfo{pages}{5773--5780}.
\newblock


\bibitem[Li and Pearl(2022b)]%
        {li2022unit}
\bibfield{author}{\bibinfo{person}{Ang Li} {and} \bibinfo{person}{Judea
  Pearl}.} \bibinfo{year}{2022}\natexlab{b}.
\newblock \showarticletitle{Unit selection with causal diagram}. In
  \bibinfo{booktitle}{\emph{Proceedings of the AAAI conference on artificial
  intelligence}}, Vol.~\bibinfo{volume}{36}. \bibinfo{pages}{5765--5772}.
\newblock


\bibitem[Li and Pearl(2024)]%
        {li2024probabilities}
\bibfield{author}{\bibinfo{person}{Ang Li} {and} \bibinfo{person}{Judea
  Pearl}.} \bibinfo{year}{2024}\natexlab{}.
\newblock \showarticletitle{Probabilities of causation with nonbinary treatment
  and effect}. In \bibinfo{booktitle}{\emph{Proceedings of the AAAI Conference
  on Artificial Intelligence}}, Vol.~\bibinfo{volume}{38}.
  \bibinfo{pages}{20465--20472}.
\newblock


\bibitem[Li et~al\mbox{.}(2023b)]%
        {DBLP:journals/tkde/LiWZZ23}
\bibfield{author}{\bibinfo{person}{Haoyang Li}, \bibinfo{person}{Xin Wang},
  \bibinfo{person}{Ziwei Zhang}, {and} \bibinfo{person}{Wenwu Zhu}.}
  \bibinfo{year}{2023}\natexlab{b}.
\newblock \showarticletitle{{OOD-GNN:} Out-of-Distribution Generalized Graph
  Neural Network}.
\newblock \bibinfo{journal}{\emph{{IEEE} Trans. Knowl. Data Eng.}}
  \bibinfo{volume}{35}, \bibinfo{number}{7} (\bibinfo{year}{2023}),
  \bibinfo{pages}{7328--7340}.
\newblock


\bibitem[Li et~al\mbox{.}(2022b)]%
        {li2022learning}
\bibfield{author}{\bibinfo{person}{Haoyang Li}, \bibinfo{person}{Ziwei Zhang},
  \bibinfo{person}{Xin Wang}, {and} \bibinfo{person}{Wenwu Zhu}.}
  \bibinfo{year}{2022}\natexlab{b}.
\newblock \showarticletitle{Learning invariant graph representations for
  out-of-distribution generalization}.
\newblock \bibinfo{journal}{\emph{Advances in Neural Information Processing
  Systems}}  \bibinfo{volume}{35} (\bibinfo{year}{2022}),
  \bibinfo{pages}{11828--11841}.
\newblock


\bibitem[Li et~al\mbox{.}(2023a)]%
        {DBLP:conf/nips/0001C0SH023}
\bibfield{author}{\bibinfo{person}{Zijian Li}, \bibinfo{person}{Ruichu Cai},
  \bibinfo{person}{Guangyi Chen}, \bibinfo{person}{Boyang Sun},
  \bibinfo{person}{Zhifeng Hao}, {and} \bibinfo{person}{Kun Zhang}.}
  \bibinfo{year}{2023}\natexlab{a}.
\newblock \showarticletitle{Subspace Identification for Multi-Source Domain
  Adaptation}. In \bibinfo{booktitle}{\emph{Advances in Neural Information
  Processing Systems 36: Annual Conference on Neural Information Processing
  Systems 2023, NeurIPS 2023, New Orleans, LA, USA, December 10 - 16, 2023}},
  \bibfield{editor}{\bibinfo{person}{Alice Oh}, \bibinfo{person}{Tristan
  Naumann}, \bibinfo{person}{Amir Globerson}, \bibinfo{person}{Kate Saenko},
  \bibinfo{person}{Moritz Hardt}, {and} \bibinfo{person}{Sergey Levine}}
  (Eds.).
\newblock


\bibitem[Li et~al\mbox{.}(2022a)]%
        {li2022graphde}
\bibfield{author}{\bibinfo{person}{Zenan Li}, \bibinfo{person}{Qitian Wu},
  \bibinfo{person}{Fan Nie}, {and} \bibinfo{person}{Junchi Yan}.}
  \bibinfo{year}{2022}\natexlab{a}.
\newblock \showarticletitle{Graphde: A generative framework for debiased
  learning and out-of-distribution detection on graphs}.
\newblock \bibinfo{journal}{\emph{Advances in Neural Information Processing
  Systems}}  \bibinfo{volume}{35} (\bibinfo{year}{2022}),
  \bibinfo{pages}{30277--30290}.
\newblock


\bibitem[Li et~al\mbox{.}(2023c)]%
        {DBLP:journals/corr/abs-2311-04837}
\bibfield{author}{\bibinfo{person}{Zijian Li}, \bibinfo{person}{Zunhong Xu},
  \bibinfo{person}{Ruichu Cai}, \bibinfo{person}{Zhenhui Yang},
  \bibinfo{person}{Yuguang Yan}, \bibinfo{person}{Zhifeng Hao},
  \bibinfo{person}{Guangyi Chen}, {and} \bibinfo{person}{Kun Zhang}.}
  \bibinfo{year}{2023}\natexlab{c}.
\newblock \showarticletitle{Identifying Semantic Component for Robust Molecular
  Property Prediction}.
\newblock \bibinfo{journal}{\emph{CoRR}}  \bibinfo{volume}{abs/2311.04837}
  (\bibinfo{year}{2023}).
\newblock
\showeprint[arXiv]{2311.04837}


\bibitem[Liu et~al\mbox{.}(2022a)]%
        {liu2022graph}
\bibfield{author}{\bibinfo{person}{Gang Liu}, \bibinfo{person}{Tong Zhao},
  \bibinfo{person}{Jiaxin Xu}, \bibinfo{person}{Tengfei Luo}, {and}
  \bibinfo{person}{Meng Jiang}.} \bibinfo{year}{2022}\natexlab{a}.
\newblock \showarticletitle{Graph rationalization with environment-based
  augmentations}. In \bibinfo{booktitle}{\emph{Proceedings of the 28th ACM
  SIGKDD Conference on Knowledge Discovery and Data Mining}}.
  \bibinfo{pages}{1069--1078}.
\newblock


\bibitem[Liu et~al\mbox{.}(2022b)]%
        {10.1145/3534678.3539347}
\bibfield{author}{\bibinfo{person}{Gang Liu}, \bibinfo{person}{Tong Zhao},
  \bibinfo{person}{Jiaxin Xu}, \bibinfo{person}{Tengfei Luo}, {and}
  \bibinfo{person}{Meng Jiang}.} \bibinfo{year}{2022}\natexlab{b}.
\newblock \showarticletitle{Graph Rationalization with Environment-Based
  Augmentations}. In \bibinfo{booktitle}{\emph{Proceedings of the 28th ACM
  SIGKDD Conference on Knowledge Discovery and Data Mining}} (Washington DC,
  USA) \emph{(\bibinfo{series}{KDD '22})}. \bibinfo{publisher}{Association for
  Computing Machinery}, \bibinfo{address}{New York, NY, USA},
  \bibinfo{pages}{1069–1078}.
\newblock
\showISBNx{9781450393850}
\urldef\tempurl%
\url{https://doi.org/10.1145/3534678.3539347}
\showDOI{\tempurl}


\bibitem[Liu and Tao(2015)]%
        {liu2015classification}
\bibfield{author}{\bibinfo{person}{Tongliang Liu} {and}
  \bibinfo{person}{Dacheng Tao}.} \bibinfo{year}{2015}\natexlab{}.
\newblock \showarticletitle{Classification with noisy labels by importance
  reweighting}.
\newblock \bibinfo{journal}{\emph{IEEE Transactions on pattern analysis and
  machine intelligence}} \bibinfo{volume}{38}, \bibinfo{number}{3}
  (\bibinfo{year}{2015}), \bibinfo{pages}{447--461}.
\newblock


\bibitem[Liu et~al\mbox{.}(2023)]%
        {liu2023flood}
\bibfield{author}{\bibinfo{person}{Yang Liu}, \bibinfo{person}{Xiang Ao},
  \bibinfo{person}{Fuli Feng}, \bibinfo{person}{Yunshan Ma},
  \bibinfo{person}{Kuan Li}, \bibinfo{person}{Tat-Seng Chua}, {and}
  \bibinfo{person}{Qing He}.} \bibinfo{year}{2023}\natexlab{}.
\newblock \showarticletitle{FLOOD: A Flexible Invariant Learning Framework for
  Out-of-Distribution Generalization on Graphs}. In
  \bibinfo{booktitle}{\emph{Proceedings of the 29th ACM SIGKDD Conference on
  Knowledge Discovery and Data Mining}}. \bibinfo{pages}{1548--1558}.
\newblock


\bibitem[Miao et~al\mbox{.}(2022)]%
        {miao2022interpretable}
\bibfield{author}{\bibinfo{person}{Siqi Miao}, \bibinfo{person}{Mia Liu}, {and}
  \bibinfo{person}{Pan Li}.} \bibinfo{year}{2022}\natexlab{}.
\newblock \showarticletitle{Interpretable and generalizable graph learning via
  stochastic attention mechanism}. In \bibinfo{booktitle}{\emph{International
  Conference on Machine Learning}}. PMLR, \bibinfo{pages}{15524--15543}.
\newblock


\bibitem[Mueller et~al\mbox{.}(2022)]%
        {DBLP:conf/ijcai/MuellerLP22}
\bibfield{author}{\bibinfo{person}{Scott Mueller}, \bibinfo{person}{Ang Li},
  {and} \bibinfo{person}{Judea Pearl}.} \bibinfo{year}{2022}\natexlab{}.
\newblock \showarticletitle{Causes of Effects: Learning Individual Responses
  from Population Data}. In \bibinfo{booktitle}{\emph{Proceedings of the
  Thirty-First International Joint Conference on Artificial Intelligence,
  {IJCAI} 2022, Vienna, Austria, 23-29 July 2022}},
  \bibfield{editor}{\bibinfo{person}{Luc~De Raedt}} (Ed.).
  \bibinfo{publisher}{ijcai.org}, \bibinfo{pages}{2712--2718}.
\newblock


\bibitem[Park and Muandet(2020)]%
        {park2020measure}
\bibfield{author}{\bibinfo{person}{Junhyung Park} {and}
  \bibinfo{person}{Krikamol Muandet}.} \bibinfo{year}{2020}\natexlab{}.
\newblock \showarticletitle{A measure-theoretic approach to kernel conditional
  mean embeddings}.
\newblock \bibinfo{journal}{\emph{Advances in neural information processing
  systems}}  \bibinfo{volume}{33} (\bibinfo{year}{2020}),
  \bibinfo{pages}{21247--21259}.
\newblock


\bibitem[Pearl(2009)]%
        {pearl2009causality}
\bibfield{author}{\bibinfo{person}{Judea Pearl}.}
  \bibinfo{year}{2009}\natexlab{}.
\newblock \bibinfo{booktitle}{\emph{Causality}}.
\newblock \bibinfo{publisher}{Cambridge university press}.
\newblock


\bibitem[Pearl(2022)]%
        {pearl2022probabilities}
\bibfield{author}{\bibinfo{person}{Judea Pearl}.}
  \bibinfo{year}{2022}\natexlab{}.
\newblock \showarticletitle{Probabilities of causation: three counterfactual
  interpretations and their identification}.
\newblock In \bibinfo{booktitle}{\emph{Probabilistic and Causal Inference: The
  Works of Judea Pearl}}. \bibinfo{pages}{317--372}.
\newblock


\bibitem[Pearl et~al\mbox{.}(2000)]%
        {pearl2000models}
\bibfield{author}{\bibinfo{person}{Judea Pearl} {et~al\mbox{.}}}
  \bibinfo{year}{2000}\natexlab{}.
\newblock \showarticletitle{Models, reasoning and inference}.
\newblock \bibinfo{journal}{\emph{Cambridge, UK: CambridgeUniversityPress}}
  \bibinfo{volume}{19}, \bibinfo{number}{2} (\bibinfo{year}{2000}),
  \bibinfo{pages}{3}.
\newblock


\bibitem[Rahmani et~al\mbox{.}(2023)]%
        {DBLP:journals/tits/RahmaniBBP23}
\bibfield{author}{\bibinfo{person}{Saeed Rahmani}, \bibinfo{person}{Asiye
  Baghbani}, \bibinfo{person}{Nizar Bouguila}, {and} \bibinfo{person}{Zachary
  Patterson}.} \bibinfo{year}{2023}\natexlab{}.
\newblock \showarticletitle{Graph Neural Networks for Intelligent
  Transportation Systems: {A} Survey}.
\newblock \bibinfo{journal}{\emph{{IEEE} Trans. Intell. Transp. Syst.}}
  \bibinfo{volume}{24}, \bibinfo{number}{8} (\bibinfo{year}{2023}),
  \bibinfo{pages}{8846--8885}.
\newblock
\urldef\tempurl%
\url{https://doi.org/10.1109/TITS.2023.3257759}
\showDOI{\tempurl}


\bibitem[Sagawa et~al\mbox{.}(2019)]%
        {DBLP:journals/corr/abs-1911-08731}
\bibfield{author}{\bibinfo{person}{Shiori Sagawa}, \bibinfo{person}{Pang~Wei
  Koh}, \bibinfo{person}{Tatsunori~B. Hashimoto}, {and} \bibinfo{person}{Percy
  Liang}.} \bibinfo{year}{2019}\natexlab{}.
\newblock \showarticletitle{Distributionally Robust Neural Networks for Group
  Shifts: On the Importance of Regularization for Worst-Case Generalization}.
\newblock \bibinfo{journal}{\emph{CoRR}}  \bibinfo{volume}{abs/1911.08731}
  (\bibinfo{year}{2019}).
\newblock
\showeprint[arXiv]{1911.08731}
\urldef\tempurl%
\url{http://arxiv.org/abs/1911.08731}
\showURL{%
\tempurl}


\bibitem[Shen et~al\mbox{.}(2021)]%
        {DBLP:journals/corr/abs-2108-13624}
\bibfield{author}{\bibinfo{person}{Zheyan Shen}, \bibinfo{person}{Jiashuo Liu},
  \bibinfo{person}{Yue He}, \bibinfo{person}{Xingxuan Zhang},
  \bibinfo{person}{Renzhe Xu}, \bibinfo{person}{Han Yu}, {and}
  \bibinfo{person}{Peng Cui}.} \bibinfo{year}{2021}\natexlab{}.
\newblock \showarticletitle{Towards Out-Of-Distribution Generalization: {A}
  Survey}.
\newblock \bibinfo{journal}{\emph{CoRR}}  \bibinfo{volume}{abs/2108.13624}
  (\bibinfo{year}{2021}).
\newblock
\showeprint[arXiv]{2108.13624}
\urldef\tempurl%
\url{https://arxiv.org/abs/2108.13624}
\showURL{%
\tempurl}


\bibitem[Shingaki et~al\mbox{.}(2021)]%
        {shingaki2021identification}
\bibfield{author}{\bibinfo{person}{Ryusei Shingaki} {et~al\mbox{.}}}
  \bibinfo{year}{2021}\natexlab{}.
\newblock \showarticletitle{Identification and estimation of joint
  probabilities of potential outcomes in observational studies with covariate
  information}.
\newblock \bibinfo{journal}{\emph{Advances in Neural Information Processing
  Systems}}  \bibinfo{volume}{34} (\bibinfo{year}{2021}),
  \bibinfo{pages}{26475--26486}.
\newblock


\bibitem[Song et~al\mbox{.}(2021)]%
        {10.5555/3491440.3491832}
\bibfield{author}{\bibinfo{person}{Ying Song}, \bibinfo{person}{Shuangjia
  Zheng}, \bibinfo{person}{Zhangming Niu}, \bibinfo{person}{Zhang-Hua Fu},
  \bibinfo{person}{Yutong Lu}, {and} \bibinfo{person}{Yuedong Yang}.}
  \bibinfo{year}{2021}\natexlab{}.
\newblock \showarticletitle{Communicative Representation Learning on Attributed
  Molecular Graphs}. In \bibinfo{booktitle}{\emph{Proceedings of the
  Twenty-Ninth International Joint Conference on Artificial Intelligence}}
  (Yokohama, Yokohama, Japan) \emph{(\bibinfo{series}{IJCAI'20})}. Article
  \bibinfo{articleno}{392}, \bibinfo{numpages}{8}~pages.
\newblock
\showISBNx{9780999241165}


\bibitem[Sui et~al\mbox{.}(2022)]%
        {sui2022causal}
\bibfield{author}{\bibinfo{person}{Yongduo Sui}, \bibinfo{person}{Xiang Wang},
  \bibinfo{person}{Jiancan Wu}, \bibinfo{person}{Min Lin},
  \bibinfo{person}{Xiangnan He}, {and} \bibinfo{person}{Tat-Seng Chua}.}
  \bibinfo{year}{2022}\natexlab{}.
\newblock \showarticletitle{Causal attention for interpretable and
  generalizable graph classification}. In \bibinfo{booktitle}{\emph{Proceedings
  of the 28th ACM SIGKDD Conference on Knowledge Discovery and Data Mining}}.
  \bibinfo{pages}{1696--1705}.
\newblock


\bibitem[Sun and Saenko(2016)]%
        {sun2016deep}
\bibfield{author}{\bibinfo{person}{Baochen Sun} {and} \bibinfo{person}{Kate
  Saenko}.} \bibinfo{year}{2016}\natexlab{}.
\newblock \showarticletitle{Deep coral: Correlation alignment for deep domain
  adaptation}. In \bibinfo{booktitle}{\emph{Computer Vision--ECCV 2016
  Workshops: Amsterdam, The Netherlands, October 8-10 and 15-16, 2016,
  Proceedings, Part III 14}}. Springer, \bibinfo{pages}{443--450}.
\newblock


\bibitem[Sun et~al\mbox{.}(2022)]%
        {sun2022does}
\bibfield{author}{\bibinfo{person}{Ruoxi Sun}, \bibinfo{person}{Hanjun Dai},
  {and} \bibinfo{person}{Adams~Wei Yu}.} \bibinfo{year}{2022}\natexlab{}.
\newblock \showarticletitle{Does GNN Pretraining Help Molecular
  Representation?}
\newblock \bibinfo{journal}{\emph{Advances in Neural Information Processing
  Systems}}  \bibinfo{volume}{35} (\bibinfo{year}{2022}),
  \bibinfo{pages}{12096--12109}.
\newblock


\bibitem[Tan et~al\mbox{.}(2022)]%
        {tan2022learning}
\bibfield{author}{\bibinfo{person}{Juntao Tan}, \bibinfo{person}{Shijie Geng},
  \bibinfo{person}{Zuohui Fu}, \bibinfo{person}{Yingqiang Ge},
  \bibinfo{person}{Shuyuan Xu}, \bibinfo{person}{Yunqi Li}, {and}
  \bibinfo{person}{Yongfeng Zhang}.} \bibinfo{year}{2022}\natexlab{}.
\newblock \showarticletitle{Learning and evaluating graph neural network
  explanations based on counterfactual and factual reasoning}. In
  \bibinfo{booktitle}{\emph{Proceedings of the ACM Web Conference 2022}}.
  \bibinfo{pages}{1018--1027}.
\newblock


\bibitem[Tian and Pearl(2000)]%
        {tian2000probabilities}
\bibfield{author}{\bibinfo{person}{Jin Tian} {and} \bibinfo{person}{Judea
  Pearl}.} \bibinfo{year}{2000}\natexlab{}.
\newblock \showarticletitle{Probabilities of causation: Bounds and
  identification}.
\newblock \bibinfo{journal}{\emph{Annals of Mathematics and Artificial
  Intelligence}} \bibinfo{volume}{28}, \bibinfo{number}{1-4}
  (\bibinfo{year}{2000}), \bibinfo{pages}{287--313}.
\newblock


\bibitem[Vedantam et~al\mbox{.}(2021)]%
        {vedantam2021empirical}
\bibfield{author}{\bibinfo{person}{Ramakrishna Vedantam},
  \bibinfo{person}{David Lopez-Paz}, {and} \bibinfo{person}{David~J Schwab}.}
  \bibinfo{year}{2021}\natexlab{}.
\newblock \showarticletitle{An empirical investigation of domain generalization
  with empirical risk minimizers}.
\newblock \bibinfo{journal}{\emph{Advances in Neural Information Processing
  Systems}}  \bibinfo{volume}{34} (\bibinfo{year}{2021}),
  \bibinfo{pages}{28131--28143}.
\newblock


\bibitem[Velickovic et~al\mbox{.}(2017)]%
        {DBLP:journals/corr/abs-1710-10903}
\bibfield{author}{\bibinfo{person}{Petar Velickovic}, \bibinfo{person}{Guillem
  Cucurull}, \bibinfo{person}{Arantxa Casanova}, \bibinfo{person}{Adriana
  Romero}, \bibinfo{person}{Pietro Li{\`{o}}}, {and} \bibinfo{person}{Yoshua
  Bengio}.} \bibinfo{year}{2017}\natexlab{}.
\newblock \showarticletitle{Graph Attention Networks}.
\newblock \bibinfo{journal}{\emph{CoRR}}  \bibinfo{volume}{abs/1710.10903}
  (\bibinfo{year}{2017}).
\newblock
\showeprint[arXiv]{1710.10903}
\urldef\tempurl%
\url{http://arxiv.org/abs/1710.10903}
\showURL{%
\tempurl}


\bibitem[Von~K{\"u}gelgen et~al\mbox{.}(2021)]%
        {von2021self}
\bibfield{author}{\bibinfo{person}{Julius Von~K{\"u}gelgen},
  \bibinfo{person}{Yash Sharma}, \bibinfo{person}{Luigi Gresele},
  \bibinfo{person}{Wieland Brendel}, \bibinfo{person}{Bernhard Sch{\"o}lkopf},
  \bibinfo{person}{Michel Besserve}, {and} \bibinfo{person}{Francesco
  Locatello}.} \bibinfo{year}{2021}\natexlab{}.
\newblock \showarticletitle{Self-supervised learning with data augmentations
  provably isolates content from style}.
\newblock \bibinfo{journal}{\emph{Advances in neural information processing
  systems}}  \bibinfo{volume}{34} (\bibinfo{year}{2021}),
  \bibinfo{pages}{16451--16467}.
\newblock


\bibitem[Watson et~al\mbox{.}(2021)]%
        {watson2021local}
\bibfield{author}{\bibinfo{person}{David~S Watson}, \bibinfo{person}{Limor
  Gultchin}, \bibinfo{person}{Ankur Taly}, {and} \bibinfo{person}{Luciano
  Floridi}.} \bibinfo{year}{2021}\natexlab{}.
\newblock \showarticletitle{Local explanations via necessity and sufficiency:
  Unifying theory and practice}. In \bibinfo{booktitle}{\emph{Uncertainty in
  Artificial Intelligence}}. PMLR, \bibinfo{pages}{1382--1392}.
\newblock


\bibitem[Wu et~al\mbox{.}(2019)]%
        {wu2019simplifying}
\bibfield{author}{\bibinfo{person}{Felix Wu}, \bibinfo{person}{Amauri Souza},
  \bibinfo{person}{Tianyi Zhang}, \bibinfo{person}{Christopher Fifty},
  \bibinfo{person}{Tao Yu}, {and} \bibinfo{person}{Kilian Weinberger}.}
  \bibinfo{year}{2019}\natexlab{}.
\newblock \showarticletitle{Simplifying graph convolutional networks}. In
  \bibinfo{booktitle}{\emph{International conference on machine learning}}.
  PMLR, \bibinfo{pages}{6861--6871}.
\newblock


\bibitem[Wu et~al\mbox{.}(2024)]%
        {wu2024time}
\bibfield{author}{\bibinfo{person}{Songhua Wu}, \bibinfo{person}{Tianyi Zhou},
  \bibinfo{person}{Yuxuan Du}, \bibinfo{person}{Jun Yu}, \bibinfo{person}{Bo
  Han}, {and} \bibinfo{person}{Tongliang Liu}.}
  \bibinfo{year}{2024}\natexlab{}.
\newblock \showarticletitle{A Time-Consistency Curriculum for Learning from
  Instance-Dependent Noisy Labels}.
\newblock \bibinfo{journal}{\emph{IEEE Transactions on Pattern Analysis and
  Machine Intelligence}} (\bibinfo{year}{2024}).
\newblock


\bibitem[Wu et~al\mbox{.}(2022)]%
        {DBLP:conf/iclr/WuWZ0C22}
\bibfield{author}{\bibinfo{person}{Yingxin Wu}, \bibinfo{person}{Xiang Wang},
  \bibinfo{person}{An Zhang}, \bibinfo{person}{Xiangnan He}, {and}
  \bibinfo{person}{Tat{-}Seng Chua}.} \bibinfo{year}{2022}\natexlab{}.
\newblock \showarticletitle{Discovering Invariant Rationales for Graph Neural
  Networks}. In \bibinfo{booktitle}{\emph{The Tenth International Conference on
  Learning Representations, {ICLR} 2022, Virtual Event, April 25-29, 2022}}.
  \bibinfo{publisher}{OpenReview.net}.
\newblock


\bibitem[Xiong et~al\mbox{.}(2019)]%
        {xiong2019pushing}
\bibfield{author}{\bibinfo{person}{Zhaoping Xiong}, \bibinfo{person}{Dingyan
  Wang}, \bibinfo{person}{Xiaohong Liu}, \bibinfo{person}{Feisheng Zhong},
  \bibinfo{person}{Xiaozhe Wan}, \bibinfo{person}{Xutong Li},
  \bibinfo{person}{Zhaojun Li}, \bibinfo{person}{Xiaomin Luo},
  \bibinfo{person}{Kaixian Chen}, \bibinfo{person}{Hualiang Jiang},
  {et~al\mbox{.}}} \bibinfo{year}{2019}\natexlab{}.
\newblock \showarticletitle{Pushing the boundaries of molecular representation
  for drug discovery with the graph attention mechanism}.
\newblock \bibinfo{journal}{\emph{Journal of medicinal chemistry}}
  \bibinfo{volume}{63}, \bibinfo{number}{16} (\bibinfo{year}{2019}),
  \bibinfo{pages}{8749--8760}.
\newblock


\bibitem[Xu et~al\mbox{.}(2019)]%
        {DBLP:conf/iclr/XuHLJ19}
\bibfield{author}{\bibinfo{person}{Keyulu Xu}, \bibinfo{person}{Weihua Hu},
  \bibinfo{person}{Jure Leskovec}, {and} \bibinfo{person}{Stefanie Jegelka}.}
  \bibinfo{year}{2019}\natexlab{}.
\newblock \showarticletitle{How Powerful are Graph Neural Networks?}. In
  \bibinfo{booktitle}{\emph{7th International Conference on Learning
  Representations, {ICLR} 2019, New Orleans, LA, USA, May 6-9, 2019}}.
  \bibinfo{publisher}{OpenReview.net}.
\newblock
\urldef\tempurl%
\url{https://openreview.net/forum?id=ryGs6iA5Km}
\showURL{%
\tempurl}


\bibitem[Xu et~al\mbox{.}(2018)]%
        {xu2018representation}
\bibfield{author}{\bibinfo{person}{Keyulu Xu}, \bibinfo{person}{Chengtao Li},
  \bibinfo{person}{Yonglong Tian}, \bibinfo{person}{Tomohiro Sonobe},
  \bibinfo{person}{Ken-ichi Kawarabayashi}, {and} \bibinfo{person}{Stefanie
  Jegelka}.} \bibinfo{year}{2018}\natexlab{}.
\newblock \showarticletitle{Representation learning on graphs with jumping
  knowledge networks}. In \bibinfo{booktitle}{\emph{International conference on
  machine learning}}. PMLR, \bibinfo{pages}{5453--5462}.
\newblock


\bibitem[Yang et~al\mbox{.}(2023)]%
        {DBLP:conf/nips/YangZF0LTWW23}
\bibfield{author}{\bibinfo{person}{Mengyue Yang}, \bibinfo{person}{Yonggang
  Zhang}, \bibinfo{person}{Zhen Fang}, \bibinfo{person}{Yali Du},
  \bibinfo{person}{Furui Liu}, \bibinfo{person}{Jean{-}Francois Ton},
  \bibinfo{person}{Jianhong Wang}, {and} \bibinfo{person}{Jun Wang}.}
  \bibinfo{year}{2023}\natexlab{}.
\newblock \showarticletitle{Invariant Learning via Probability of Sufficient
  and Necessary Causes}. In \bibinfo{booktitle}{\emph{Advances in Neural
  Information Processing Systems 36: Annual Conference on Neural Information
  Processing Systems 2023, NeurIPS 2023, New Orleans, LA, USA, December 10 -
  16, 2023}}, \bibfield{editor}{\bibinfo{person}{Alice Oh},
  \bibinfo{person}{Tristan Naumann}, \bibinfo{person}{Amir Globerson},
  \bibinfo{person}{Kate Saenko}, \bibinfo{person}{Moritz Hardt}, {and}
  \bibinfo{person}{Sergey Levine}} (Eds.).
\newblock


\bibitem[Yang et~al\mbox{.}(2022)]%
        {yang2022learning}
\bibfield{author}{\bibinfo{person}{Nianzu Yang}, \bibinfo{person}{Kaipeng
  Zeng}, \bibinfo{person}{Qitian Wu}, \bibinfo{person}{Xiaosong Jia}, {and}
  \bibinfo{person}{Junchi Yan}.} \bibinfo{year}{2022}\natexlab{}.
\newblock \showarticletitle{Learning substructure invariance for
  out-of-distribution molecular representations}.
\newblock \bibinfo{journal}{\emph{Advances in Neural Information Processing
  Systems}}  \bibinfo{volume}{35} (\bibinfo{year}{2022}),
  \bibinfo{pages}{12964--12978}.
\newblock


\bibitem[Ying et~al\mbox{.}(2019)]%
        {ying2019gnnexplainer}
\bibfield{author}{\bibinfo{person}{Zhitao Ying}, \bibinfo{person}{Dylan
  Bourgeois}, \bibinfo{person}{Jiaxuan You}, \bibinfo{person}{Marinka Zitnik},
  {and} \bibinfo{person}{Jure Leskovec}.} \bibinfo{year}{2019}\natexlab{}.
\newblock \showarticletitle{Gnnexplainer: Generating explanations for graph
  neural networks}.
\newblock \bibinfo{journal}{\emph{Advances in neural information processing
  systems}}  \bibinfo{volume}{32} (\bibinfo{year}{2019}).
\newblock


\bibitem[Ying et~al\mbox{.}(2018)]%
        {ying2018hierarchical}
\bibfield{author}{\bibinfo{person}{Zhitao Ying}, \bibinfo{person}{Jiaxuan You},
  \bibinfo{person}{Christopher Morris}, \bibinfo{person}{Xiang Ren},
  \bibinfo{person}{Will Hamilton}, {and} \bibinfo{person}{Jure Leskovec}.}
  \bibinfo{year}{2018}\natexlab{}.
\newblock \showarticletitle{Hierarchical graph representation learning with
  differentiable pooling}.
\newblock \bibinfo{journal}{\emph{Advances in neural information processing
  systems}}  \bibinfo{volume}{31} (\bibinfo{year}{2018}).
\newblock


\bibitem[Zhang et~al\mbox{.}(2018)]%
        {DBLP:conf/iclr/ZhangCDL18}
\bibfield{author}{\bibinfo{person}{Hongyi Zhang}, \bibinfo{person}{Moustapha
  Ciss{\'{e}}}, \bibinfo{person}{Yann~N. Dauphin}, {and} \bibinfo{person}{David
  Lopez{-}Paz}.} \bibinfo{year}{2018}\natexlab{}.
\newblock \showarticletitle{mixup: Beyond Empirical Risk Minimization}. In
  \bibinfo{booktitle}{\emph{6th International Conference on Learning
  Representations, {ICLR} 2018, Vancouver, BC, Canada, April 30 - May 3, 2018,
  Conference Track Proceedings}}. \bibinfo{publisher}{OpenReview.net}.
\newblock
\urldef\tempurl%
\url{https://openreview.net/forum?id=r1Ddp1-Rb}
\showURL{%
\tempurl}


\bibitem[Zhang et~al\mbox{.}(2022a)]%
        {zhang2022partial}
\bibfield{author}{\bibinfo{person}{Junzhe Zhang}, \bibinfo{person}{Jin Tian},
  {and} \bibinfo{person}{Elias Bareinboim}.} \bibinfo{year}{2022}\natexlab{a}.
\newblock \showarticletitle{Partial counterfactual identification from
  observational and experimental data}. In
  \bibinfo{booktitle}{\emph{International Conference on Machine Learning}}.
  PMLR, \bibinfo{pages}{26548--26558}.
\newblock


\bibitem[Zhang et~al\mbox{.}(2022b)]%
        {zhang2022multi}
\bibfield{author}{\bibinfo{person}{Weijia Zhang}, \bibinfo{person}{Xuanhui
  Zhang}, \bibinfo{person}{Min-Ling Zhang}, {et~al\mbox{.}}}
  \bibinfo{year}{2022}\natexlab{b}.
\newblock \showarticletitle{Multi-instance causal representation learning for
  instance label prediction and out-of-distribution generalization}.
\newblock \bibinfo{journal}{\emph{Advances in Neural Information Processing
  Systems}}  \bibinfo{volume}{35} (\bibinfo{year}{2022}),
  \bibinfo{pages}{34940--34953}.
\newblock


\bibitem[Zhang et~al\mbox{.}(2021)]%
        {zhang2021deep}
\bibfield{author}{\bibinfo{person}{Xingxuan Zhang}, \bibinfo{person}{Peng Cui},
  \bibinfo{person}{Renzhe Xu}, \bibinfo{person}{Linjun Zhou},
  \bibinfo{person}{Yue He}, {and} \bibinfo{person}{Zheyan Shen}.}
  \bibinfo{year}{2021}\natexlab{}.
\newblock \showarticletitle{Deep stable learning for out-of-distribution
  generalization}. In \bibinfo{booktitle}{\emph{Proceedings of the IEEE/CVF
  Conference on Computer Vision and Pattern Recognition}}.
  \bibinfo{pages}{5372--5382}.
\newblock


\bibitem[Zhao et~al\mbox{.}(2020)]%
        {zhao2020uncertainty}
\bibfield{author}{\bibinfo{person}{Xujiang Zhao}, \bibinfo{person}{Feng Chen},
  \bibinfo{person}{Shu Hu}, {and} \bibinfo{person}{Jin-Hee Cho}.}
  \bibinfo{year}{2020}\natexlab{}.
\newblock \showarticletitle{Uncertainty aware semi-supervised learning on graph
  data}.
\newblock \bibinfo{journal}{\emph{Advances in Neural Information Processing
  Systems}}  \bibinfo{volume}{33} (\bibinfo{year}{2020}),
  \bibinfo{pages}{12827--12836}.
\newblock


\end{thebibliography}

\newpage
\appendix
\onecolumn

\section{Proof of Theorem~\ref{thm:bound}}\label{sec:proof_bound}
\begin{proof}
    To find the lower bound of PNS, by \emph{Bonferroni's inequality}, for any three events $A$, $B$, we have the bounds
\begin{equation}
    \small
         P(A, B) \geq \max (0, P(A)+P(B) - 1) 
\end{equation}
We substitute $A$ for $Y_{do(C=c)}=y$ and $B$ for $Y_{do(C\ne c)} \ne y$. Then, we have
\begin{equation}\label{equ:l1}
\small
\begin{aligned}
    &P(Y_{do(C=c)}=y, Y_{do(C\ne c)} \ne y)\\
    &\geq \max (0, P(Y_{do(C=c)}=y)+P(Y_{do(C\ne c)} \ne y) - 1) 
\end{aligned}
\end{equation}
Based on the assumption of our causal model, the condition that there is no confounding between variables $C$ and $Y$ is met. Thus, the intervention probability $P(Y_{do(C=c)}=y)$ and  $P(Y_{do(C\ne c)} \ne y)$ can be identified by conditional probabilities $P(Y_{do(C=c)}=y)$ and  $P(Y_{do(C\ne c)} \ne y)$, respectively. Then, Eq.~\ref{equ:l1} can be rewritten as 
\begin{equation}\label{equ:asfgd}
\small
\begin{aligned}
    &P(Y_{do(C=c)}=y, Y_{do(C\ne c)} \ne y)\\
    &\geq \max (0, P(Y=y|C=c)+P(Y \ne y| C \ne c) - 1)\\
    &= \max (0, P(Y=y|C=c)- P(Y = y| C \ne c))
\end{aligned} 
\end{equation}
We now prove that $P(Y_{do(C=c)}=y, Y_{do(C\ne c)} \ne y)$ is equivalent to PNS. Specifically, according to the consistency  assumption of counterfactual reasoning, i.e., 
$(C = c) \Rightarrow(Y_{do(C=c)}=y)= (Y = y)$, 
$(C \ne c) \Rightarrow(Y_{do(C\ne c)}\ne y)= (Y \ne y)$, we know that 
\begin{equation}\label{equ:l2}
\small
\begin{aligned}
    &(Y_{do(C=c)}=y) \wedge (Y_{do(C\ne c)} \ne y)\\
    =&\big((Y_{do(C=c)}=y) \wedge (Y_{do(C\ne c)} \ne y )\big) \wedge \big( (C =c) \vee (C \ne c) \big)\\
    =&\big((Y_{do(C=c)}=y) \wedge (Y_{do(C\ne c)} \ne y )  \wedge (C=c)  \big) \vee\\ &\big( (Y_{do(C=c)}=y) \wedge (Y_{do(C\ne c)} \ne y ) \vee (C \ne c) \big)\\
    =&\big((Y=y) \wedge (Y_{do(C\ne c)} \ne y )  \wedge (C=c)  \big) \vee\\ &\big( (Y_{do(C=c)}=y) \wedge (Y \ne y ) \vee (C \ne c) \big)
\end{aligned}
\end{equation}

By Eq.~\ref{equ:l2}, we have
\begin{equation}\label{equ:ddqqtr}
\small
\begin{aligned}
&P(Y_{do(C=c)}=y, Y_{do(C\ne c)} \ne y)\\
=&P(Y_{do(C=c)}=y, C \ne c, Y \ne y) + P(Y_{do(C\ne c)} \ne y, C=c, Y=y)\\
=&P(Y_{do(C=c)}=y | C \ne c, Y \ne y)P(C \ne c, Y \ne y)\\
&+ P(Y_{do(C\ne c)} \ne y| C=c, Y=y)P(C=c, Y=y),
\end{aligned}
\end{equation}
which is the definition of $\text{PNS}(C=c, Y=y)$ (Definition \ref{def:pns}).

By  Eq.~\ref{equ:asfgd} and Eq.~\ref{equ:ddqqtr}, we have 
\begin{equation}
\small
   \text{PNS}(C=c, Y=y) \geq \max (0, P(Y=y|C=c)- P(Y = y| C \ne c))
\end{equation}


Then, since conditional distribution $P(Y=y|C \ne c)$  can be written as
\begin{equation*}
\small
        \begin{aligned}
            &P(Y=y|C \ne c)\\
            &= \frac{P(Y=y) - P(C=c,Y=y)}{1 - P(C = c)} \quad(\text{Bayes' Rule})\\
            &=  \frac{P(Y=y) - P(C=c)P(Y=y|C=c)}{1 - P(C = c)} \\
            &= \frac{P(Y=y) - \sum_{G} P(G, C=c) P(Y=y|C=c)}{1 - \sum_G P(G, C = c)} \quad(\text{Law of Total Probability})\\
            &=\frac{P(Y=y) - \mathbb E_G[P(C=c|G)] P(Y=y|C=c)}{1 - \mathbb E_G [P(C = c|G)]}
        \end{aligned}
    \end{equation*}
we have
\begin{equation*}
\small
\begin{aligned}
    &P(Y = y|C =c) - P(Y = y|C \ne c)\\
    &P(Y = y|C =c)  -  \frac{P(Y=y) - \mathbb E_G[P(C=c|G)] P(Y=y|C=c)}{1 - \mathbb E_G [P(C = c|G)]}\\
    &= P(Y = y|C =c) +  P(Y=y|C=c) \cdot \frac{ \mathbb E_G[P(C=c|G)]}{1 - \mathbb E_G [P(C=c|G)]}  - \frac{ P(Y=y)}{1 - \mathbb E_G [P(C=c|G)]}\\
    &= \frac{ P(Y = y|C =c) - P(Y=y)}{1 - \mathbb E_G [P(C=c|G)]} 
\end{aligned}
\end{equation*}
\end{proof}

\section{Proof of Theorem~\ref{lma:boost}}\label{sec:proof_boost}
Before proving Theorem~\ref{lma:boost}, we prove Lemma~\ref{lma:nonzero}, which allows us to safely divide by the quantity 
\begin{lemma}\label{lma:nonzero}
    In the setting of Theorem~\ref{lma:boost}, $P(\hat{Y} =1 |Y =1) + P(\hat{Y} =0 |Y =0) = 1$ if  and only if $C$ and $Y$ are independent.
\end{lemma}
\begin{proof}
Let $P(\hat{Y} =1 |Y =1) = \epsilon_1$ and $P(\hat{Y} =0 |Y =0) = \epsilon_0$.
We first prove the forward implication. Suppose $\epsilon_1 + \epsilon_0 = 1$, we have 
\begin{equation}
\small
    \begin{aligned}
\mathbb{E}[\hat{Y}] & =\epsilon_1 \operatorname{Pr}[Y=1]+\left(1-\epsilon_0\right)(1-\operatorname{Pr}[Y=1])~(\text{Law of Total Expectation}) \\
& =\left(\epsilon_0+\epsilon_1-1\right) \operatorname{Pr}[Y=1]+1-\epsilon_0 \\
& =1-\epsilon_0~(P(\hat{Y} =1 |Y =1) + P(\hat{Y} =0 |Y =0) = 1)\\
& =\mathbb{E}[\hat{Y} \mid Y=0] .
\end{aligned}
\end{equation}
Since $Y$ and $\hat{Y}$ are both binary, and the distribution of $\hat{Y}$ is specified entirely by its mean, we have $\hat{Y} \perp Y$. Further, the covariance between $\hat{Y}$ and $Y$ is 0, i.e., 
\begin{equation}
\small
    \begin{aligned}
0 & =\mathbb{E}[(Y-\mathbb{E}[Y])(\widehat{Y}-\mathbb{E}[\widehat{Y}])] \\
& =\mathbb{E}\left[\mathbb{E}\left[(Y-\mathbb{E}[Y])(\widehat{Y}-\mathbb{E}[\widehat{Y}]) \mid X_S\right]\right] \\
& =\mathbb{E}\left[\mathbb{E}\left[Y-\mathbb{E}[Y] \mid X_S\right] \mathbb{E}\left[\widehat{Y}-\mathbb{E}[\widehat{Y}] \mid X_S\right]\right] \\
& =\mathbb{E}\left[\left(\mathbb{E}\left[Y-\mathbb{E}[Y] \mid X_S\right]\right)^2\right],
\end{aligned}
\end{equation}
so we have $\mathbb E[Y |X_S] = \mathbb E[Y]$. Since $Y$ is binary, $\mathbb E[Y |X_S] = P(Y=1|X_S)$ and $\mathbb E[Y]= P(Y=1)$ and we have 
$P(Y=1|X_S) = P(Y=1)$ and $P(Y=0|X_S) = P(Y=0)$, that is, $P(Y|X_S) = P(Y)$. Thus, $Y \perp S$. 

To prove the reverse implication, suppose $Y \perp S$, we have
\begin{equation}
\small
\epsilon_1=P(\hat{Y}=1 \mid Y=1) =P(\hat{Y}=1) =1-P(\hat{Y}=0) =1-P(\hat{Y}=0 \mid Y=0) =1-\epsilon_0
\end{equation}
\end{proof}
We now use Lemma~\ref{lma:nonzero} to prove Theorem~\ref{lma:boost}:
\begin{proof}
   Given $C$, $Y$ and $\hat Y$ have the same conditional distribution, i.e., 
\begin{equation}\label{equ:yy1}
\small
    P(Y|C) =P(\hat Y|C)
\end{equation}
Next we prove $P(Y = 1) = P(\hat{Y} = 1)$. 
\begin{equation}\label{equ:yy2}
\small
\begin{aligned}
    P(\hat Y) &= \sum_{C} P(\hat Y, C)= \sum_{C} P(C)P(\hat Y | C)\\ 
    &= \mathbb E_{C}[ P(\hat Y|C) ]= \mathbb E_{C}[ P( Y|C) ] = P(Y)   
\end{aligned}
\end{equation}

Based on the above equalities, we first derive the calculation formulas for $P(\hat Y=1| Y=1)$ and $P(\hat Y=0| Y=0)$.
\begin{equation}
\small
\begin{aligned}
&P(\hat Y=1| Y=1) = \frac{P(\hat Y=1, Y=1)}{P(Y=1)}\\
&= \frac{P(\hat Y=1, Y=1)}{P(\hat{Y}=1)}~\text{(according to Eq.~\ref{equ:yy2})}\\
&= \frac{\sum_{C} P(\hat Y=1, Y=1, C)}{\sum_{C} P(\hat Y=1, C)}\\
&= \frac{\mathbb E_{C} [P(\hat Y=1, Y=1| C)]}{\mathbb E_{C} [P(\hat Y=1 | C)]}\\
&= \frac{\mathbb E_{C} [P(\hat Y=1| C) \cdot P(Y=1| C)]}{\mathbb E_{C} [P(\hat Y=1 | C)]}\\
&= \frac{\mathbb E_{C} [ P(Y=1| C)^2]}{\mathbb E_{C} [P(Y=1 | C)]}~\text{(according to Eq.~\ref{equ:yy1})}
\end{aligned}
\end{equation}
Similarly, we can obtain 
\begin{equation}
    P(\hat Y=0| Y=0) = \frac{\mathbb E_{C} [ P(Y=0| C)^2]}{\mathbb E_{C} [P(Y=0 | C)]}
\end{equation}
Next we will discuss the connection between $P(\hat Y=1|S, E)$ and $P(Y=1|S, E)$ by expanding $P(\hat Y=1|S, E)$.
\begin{equation}
\small
\begin{aligned}
    &P(\hat Y=1|S, E)  = \frac{P(\hat Y=1, S, E) }{P(S, E)}\\
    =&\frac{P(\hat Y=1, S, E,Y=1) }{P(S, E)} + \frac{P(\hat Y=1, S, E, Y=0) }{P(S, E)} ~(\text{Law of Total Probability})\\
    =&P(Y=1|S, E) \cdot P(\hat Y =1| S, E,Y=1) +P(Y=0|S, E) \cdot P(\hat Y =1|S, E,Y=0)\\
    =&P(Y=1|S, E) \cdot P(\hat Y =1| Y=1)+P(Y=0|S, E) \cdot P(\hat Y =1| Y=0)\text{($C \perp \{S,E\}|Y$, and $C \to \hat Y$)}\\
    =&P(Y=1|S,E)(P(\hat Y = 1|Y=1) + P(\hat Y=0|Y=0) - 1 ) + 1 - P(\hat Y = 0|Y = 0)
\end{aligned}
\end{equation}
Combining the conclusion of Lemma~\ref{lma:nonzero}, we can get $P(\hat Y=1|Y=1) + P(\hat Y=0|Y=0) \ne 1$, so we have
\begin{equation}
\small
    P(Y=1|S,E) = \frac{P(\hat Y = 1 |S,E) -1 +  P(\hat Y=0|Y=0)}{P(\hat Y=1|Y=1)   + P(\hat Y=0|Y=0)- 1}
\end{equation}
Now we have $P(Y|S,E)$ and $P(Y|C)$. Our next question is how to combine them to calculate $P(Y|S, C, E)$. Specifically, we first calculate the Odds of $P(Y|S, C, E)$.
\begin{equation}\label{equ:safavg}
\small
    \begin{aligned}
        &\frac{P(Y=1|S, C, E)}{P(Y=0|S, C, E)} = \frac{P(Y=1)P(S, C, E|Y=1)}{P(Y=0)P(S, C, E|Y=0)}\\
        &=\frac{P(Y=1)P(S,E |Y=1)P( C|Y=1)}{P(Y=0)P(S, E |Y=0)P( C|Y=0)}~\text{($C \perp \{S, E\}|Y$)}\\
        &= \frac{P(Y=1)P(Y=1|C)P( Y=1|S,E)}{P(Y=0)P(Y=0|C)P(Y=0| S,E)} \\
    \end{aligned}
\end{equation}
Applying the log function to each side of Eq.~\ref{equ:safavg}, we have
\begin{equation}
\begin{aligned}
\small
    \log \frac{P(Y=1|S, C,E)}{P(Y=0|S, C, E)} = &\log \frac{P(Y=1)}{P(Y=0)} + \log \frac{P(Y=1|C)}{P(Y=0|C)} + \log \frac{P(Y=1|S,E)}{P(Y=0|S,E)}
\end{aligned}
\end{equation}
The aforementioned formula can also be expressed as
\begin{equation}
\small
\begin{aligned}
    &\text{logit}(P(Y=1|S, C, E)) =\\
    &\text{logit}(P(Y=1)) + \text{logit}(P(Y=1|C)) + \text{logit}(P(Y=1|S,E))
\end{aligned}
\end{equation}
Since sigmoid function (denoted as $\sigma$) is the inverse function of logit, we have
\begin{equation}
\small
\begin{aligned}
    &P(Y=1|S, C, E) = \sigma(\text{logit}(P(Y=1|S, C, E)) )\\  
    &= \sigma(\text{logit}(P(Y=1))  + \text{logit}(P(Y=1|C)) + \text{logit}(P(Y=1|S,E)))
\end{aligned}
\end{equation}
\end{proof}

\section{Multiclass Case}\label{sec:multi_class}
In the main text of the paper, we employed a simplified notation to present our test domain adaptation method in the context of binary labels $Y$. However, in numerous instances, including our experiments in Section~\ref{sec:exp}, the label $Y$ can have more than two classes. Consequently, in this section, we illustrate the means of extending our method to the multiclass setting. 

\begin{equation}\label{equ:asfagg}
\small
\begin{aligned}
        &P(\hat{Y}=y|S,E)\\
        &= \sum_{y' \in \mathcal{Y}} P(\hat{Y}=y|S,E, Y=y')P(Y=y'|S,E)\\
        &= \sum_{y' \in \mathcal{Y}} P(\hat{Y}=y| Y=y')P(Y=y'|S,E)~\text{($C \perp \{S,E\}|Y$, and $\hat Y$ is determined by the $C$)}\\
\end{aligned}
\end{equation}
Let $M \in [0, 1]^{\mathcal{Y} \times \mathcal{Y}}$ ($\mathcal{Y} = \{1, 2,..., |\mathcal{Y}|\}$) denote a  matrix  with 
\begin{equation}
\small
    M_{ij} = P(\hat Y=i|Y=j).
\end{equation}
Let $h \in [0, 1]^\mathcal{V}$ denote a vector with 
\begin{equation}
\small
    h_i = P(Y=i|S, E).
\end{equation}
In matrix notation, Eq.~\ref{equ:asfagg} can be seen as a 
\begin{equation}
\small
    P(\hat{Y}|S,E) = M \cdot h \in [0, 1]^\mathcal{Y}
\end{equation}
When $M$ is non-singular, we can calibrate $P(\hat{Y}=y|S,E)$ using the following equality.
\begin{equation}
\small
    P(Y|S, E) = M^{-1} \cdot P(\hat{Y}|S,E)
\end{equation}

\begin{equation}
\small
\begin{aligned}
    P(\hat{Y}=y| Y=y') 
    &= \frac{P(\hat{Y}=y, Y=y')}{P(Y=y')}\\
    &= \frac{\mathbb E_{C} [P(\hat{Y}=y, Y=y'|C)]}{\mathbb E_{C} [P(Y=y'|C)]}\\
    &= \frac{\mathbb E_{C} [P(Y=y'|C) P(\hat{Y}=y|C,Y)]}{\mathbb E_{C} [P(Y=y'|C)]}\\
    &= \frac{\mathbb E_{C} [P(Y=y'|C) P(\hat{Y}=y|C)]}{\mathbb E_{C} [P(Y=y'|C)]}~\text{($\hat Y$ is determined by the $C$)}\\
\end{aligned}
\end{equation}

Similarly, we start by calculating the Odds of $P(Y=y|X_S, X_U, E)$.
\begin{equation}\label{equ:sdavagg}
\small
\begin{aligned}
    \frac{P(Y=y|C, S, E)}{P(Y \ne y|C, S, E)} 
    &= \frac{P(Y=y|C, S, E)}{\sum_{y' \ne y} P(Y = y'|C, S, E)}\\
    &= \frac{P(Y=y,C, S, E)}{\sum_{y' \ne y} P(Y = y',C, S, E)}\\
    &= \frac{P(Y=y)P(C, S, E|Y=y)}{\sum_{y' \ne y} P(Y = y')P(C, S, E|Y=y')}\\
    &=\frac{P(Y=y)P(C|Y=y)P(S, E|Y=y)}{\sum_{y' \ne y} P(Y = y')P( C|Y=y')P(S,E |Y=y')}~(C \perp \{S,E\}|Y)\\
    &= \frac{P(Y=y)P(Y=y|C)P(Y=y|S, E)}{\sum_{y' \ne y} P(Y = y')P( Y=y'|C)P(Y=y'|S, E)}\\
\end{aligned}
\end{equation}

Let 
\begin{equation}
\small
    Q_y = P(Y=y)P(Y=y|C)P(Y=y|S, E) ~\text{for each $y \in \mathcal{Y}$},  
\end{equation} 
and 
\begin{equation}
\small
    \|Q\|_1 = \sum_{y \in \mathcal{Y}} Q_y.
\end{equation}
Applying the log function to each side of Eq.~\ref{equ:sdavagg}, we have
\begin{equation}
\small
\begin{aligned}
    &\text{logit}(P(Y=y|C, S, E)) \\
    &= \log \frac{P(Y=y|C, S, E)}{P(Y \ne y|C, S, E)}\\
    &= \log \frac{Q_y}{\sum_{y' \ne y} Q_{y'}} = \log \frac{\frac{Q_y}{\|Q\|_1}}{\sum_{y' \ne y} \frac{Q_{y'}}{\|Q\|_1}} = \text{logit}\Big(\frac{Q_y}{\|Q\|_1}\Big)
\end{aligned}
\end{equation}
Note that the reason we convert $\log \frac{Q_y}{\sum_{y' \ne y} Q_{y'}}$ into $\log \frac{\frac{Q_y}{\|Q\|_1}}{\sum_{y' \ne y} \frac{Q_{y'}}{\|Q\|_1}}$ is that $Q_y + \sum_{y' \ne y} Q_{y'} \ne 1$ and cannot be converted into logit.

Since sigmoid function $\sigma$ is the inverse function of logit, we have
\begin{equation}
\small
    P(Y=y|C, S, E) = \sigma(\text{logit}(P(Y=y|C, S, E))) = \sigma\Big(\text{logit}\Big(\frac{Q_y}{\|Q\|_1}\Big)\Big) = \frac{Q_y}{\|Q\|_1}
\end{equation}

\end{document}